\documentclass{article}

\PassOptionsToPackage{numbers,sort&compress}{natbib}

 \usepackage[preprint]{neurips_2026}

\usepackage[utf8]{inputenc} % allow utf-8 input
\usepackage[T1]{fontenc}    % use 8-bit T1 fonts
\usepackage{hyperref}       % hyperlinks
\usepackage{url}            % simple URL typesetting
\usepackage{booktabs}       % professional-quality tables
\usepackage{amsfonts}       % blackboard math symbols
\usepackage{nicefrac}       % compact symbols for 1/2, etc.
\usepackage{microtype}      % microtypography
\usepackage{xcolor}         % colors

\usepackage{amsmath}
\usepackage{amssymb}
\usepackage{mathtools}
\usepackage{amsthm}
\usepackage{ulem}
\usepackage[table]{xcolor}
\usepackage{multirow}
\usepackage{xurl}
\definecolor{ustcblue}{RGB}{0,82,204}    % RGB 0-255
\usepackage{algorithm}
\usepackage{algorithmic}
\usepackage{graphicx} % DO NOT CHANGE THIS
\usepackage{wrapfig}

\theoremstyle{plain}
\newtheorem{theorem}{Theorem}[section]
\newtheorem{proposition}[theorem]{Proposition}
\newtheorem{lemma}[theorem]{Lemma}
\newtheorem{corollary}[theorem]{Corollary}
\theoremstyle{definition}
\newtheorem{definition}[theorem]{Definition}

\theoremstyle{remark}

\title{CSGuard: Toward Forgery-Resistant Watermarking in Diffusion Models via Compressed Sensing Constraint}

\author{Jiewei Lai$^{1}$, Lan Zhang$^{1,2}$\thanks{Corresponding author.}\texttt{ }, Chen Tang$^{1}$, Pengcheng Sun$^{1}$, Zhaopeng Zhang$^{1}$,\\
\textbf{Yunhao Wang$^{3}$,Hui Jin$^{3}$} \\
$^{1}$University of Science and Technology of China, 
$^{2}$Institute of Artificial Intelligence,\\ Hefei Comprehensive National Science Center ,
$^{3}$Lenovo Research \\
}

\begin{document}

\maketitle

\begin{abstract}
Latent-based diffusion model watermarking embeds watermarks into generated images' latent space to enable content attribution, offering a training-free solution for intellectual property protection and digital forensics. 
However, these methods exhibit a critical vulnerability to the forgery attack, attackers can extract the watermark by inverting the watermarked image and re-generating it with an arbitrary prompt, thereby enabling false attribution on malicious content.
In this paper, we propose the CSGuard, the first forgery-resistant watermarking schema that leverages compressed sensing to bind the watermarked image generation and verification to a secret matrix.
This ensures that only users possessing the secret matrix can correctly embed or verify the image watermark, prevents the illegal users from forgery without compromising generation quality and watermark integrity.
Experimental results demonstrate that CSGuard achieves strong forgery resistance, reduces the attack success rate from 100.0\% to 28.12\%, and achieve 100\% detection rate on benign watermarked images without compromising watermarking effectiveness.
% The code is available at \textbf{\textit{https://anonymous.4open.science/r/CSGuard-5740}}.
\end{abstract}

\section{Introduction}
% 研究背景
In recent years, diffusion models (DMs) have witnessed remarkable progress and widespread adoption due to their exceptional generative capabilities \citep{DBLP:conf/cvpr/RombachBLEO22,DBLP:conf/iclr/ChenYGYXWK0LL24,DBLP:conf/iclr/PodellELBDMPR24,DBLP:journals/csur/YangZSHXZZCY24}. These models enable low-cost synthesis of diverse and highly realistic images, bringing transformative opportunities but also significant societal risks. In particular, the ease of access to such powerful generative tools raises serious concerns regarding potential misuse and model copyright infringement, ranging from unauthorized replication of proprietary models to the creation of deceptive or malicious content that could fuel misinformation campaigns \citep{DBLP:journals/tdsc/ZhuLWZWL25, DBLP:conf/icml/CiS0XS25,DBLP:journals/corr/abs-2508-03067,DBLP:journals/corr/abs-2506-23707}.
% 换引用？

% 引出水印
To address these issues,  diffusion models watermarking (DMsMark) have emerged as a promising solution. By subtly modifying either model weights (model-based) or the image generation process (latent-based), such methods embed imperceptible signals into generated images, thereby enabling verifiable attribution of origin and ownership through watermark detection. Among these approaches, latent-based DMsMark, which embeds watermarks in the latent space during generation and verifies them without altering the model, has become a dominant paradigm \citep{DBLP:conf/nips/WenKGG23,DBLP:conf/cvpr/YangZCF0Y24,DBLP:conf/icml/LiHHH25,DBLP:journals/csur/DengLZLPWS26}.
%引用++其他等写相关工作时补齐

% 现有不足
However, latent-based DMsMark methods exhibit a critical vulnerability to forgery attacks (specifically, reprompt attacks) \citep{DBLP:conf/cvpr/0025LTFQ25}, which severely undermines their security and practical reliability. Specifically, as shown in Fig~\ref{fig:csgurad}, an adversary can leverage any off-the-shelf DM to invert a watermarked image (named a benign image) to recover its initial latent variable containing the embedded user watermark through inversion. By reusing this latent variable with a malicious prompt, the attacker can generate a forged image that appears semantically unrelated to the original content. Crucially, during verification, the watermark extracted from this forged image will still be attributed to the legitimate user, leading to false attribution of malicious content. This not only compromises the reliability of the watermarking system but also risks implicating innocent users in harmful activities.

% 问题定义和研究挑战
Therefore, this work aims to develop a latent-based DMsMark mechanism that simultaneously resists forgery attacks while preserving dual fidelity, including watermarking effectiveness and image quality.
% without compromising watermarking effectiveness. 
% The vulnerability of latent-based DMsMark to forgery attacks stems from an intrinsic symmetry inherent in DMs: (1) the approximate invertibility between the generation and the inversion process, and (2) the deterministic mapping between pixel space and latent space via the autoencoder~\citep{DBLP:conf/cvpr/RombachBLEO22,DBLP:conf/nips/HuCWLWSL23,DBLP:conf/nips/HangMCFXFZW24}.
The vulnerability of latent-based DMsMark to forgery attacks stems from an intrinsic symmetry inherent in DMs, which we term the generation–inversion symmetry, which comprises two key properties: (1) the approximate invertibility between the generation and the inversion process, and (2) the deterministic mapping between pixel space and latent space via the autoencoder~\citep{DBLP:conf/cvpr/RombachBLEO22,DBLP:conf/nips/HuCWLWSL23,DBLP:conf/nips/HangMCFXFZW24}.
This symmetry enables the feasibility of watermark embedding and extraction, however, it also constitutes the fundamental critical vulnerability that enables forgery: an adversary can invert a watermarked image using any off-the-shelf DMs to recover the latent representation containing the watermark. Consequently, this structural symmetry, which is central to the functioning of DMsMark, also renders these schemes inherently susceptible to forgery attacks. This inherent conflict between watermarking effectiveness and  anti-forgery constitutes the core challenge in designing such a DMsMark scheme.

% design
% 为了解决这个问题，在这篇论文中，我们提出了CSGuard，据我们所知，这是第一个能够抵抗伪造攻击的DMsMark，能够在保障图像生成质量和水印性能的情况下，抵抗伪造攻击。
% 正如上文所述，The core vulnerability of latent-based DMsMark stems from
% its reliance on the intrinsic symmetry of DMs，因此抗伪造攻击的关键是打破这种对称性。
% 受压缩感知的启发，where only parties possessing the sampling matrix can reconstruct the original signal from its compressed observation，我们提出binding the generation–inversion symmetry to a secret matrix，只有合法的用户即拥有秘密矩阵，才能够利用这种对称性进行带水印图像的生成和验证。
% 具体地，我们通过让生成过程的中间图像的latent表征满足CS一致性从而将秘密矩阵A嵌入到生成过程过程中。再进行验证时，只有让中间图像的latent表征满足统一的CS一致性约束才能够成功地得到初始携带有水印信息的潜变量，这使得攻击者没有A无法获成功进行反演和进行伪造攻击；
% 为了保障DMsMark的图像生成和水印能力，我们利用通过最小扰动投影和“内生”观测值构建，使得投影后的变量既满足一致性约束，又不会偏离原有的生成轨迹，有效得实现了安全性和忠诚性的平衡，在维持原有DMsMark情况下有效抵抗伪造攻击。

To address this challenge, we propose CSGuard, the first DMsMark scheme that is resistant to the forgery attack while preserving both generation quality and watermark effectiveness. 
% through compressed sensing (CS) constraint.
The core vulnerability of latent-based DMsMark lies in its reliance on the generation–inversion symmetry, therefore, the key insight of CSGuard is to break this symmetry. 
We leverage compressed sensing to construct a consistency constraint to enforce that the latent representation of intermediate images satisfies the $\mathbf{A}z_{0|t} \equiv \mathbf{y}$ (named CS constraint), which binds the generation–inversion symmetry to a secret matrix $\mathbf{A}$ during generation.
During verification, the initial watermarked latent variable is recoverable only if the same CS consistency holds across the inversion trajectory. Consequently, an attacker without access to $\mathbf{A}$ cannot perform a valid inversion, thereby preventing forgery.
To ensure minimal impact on generation quality and watermark performance, we further design a minimum-perturbation projection with trajectory-intrinsic observation construction. We derive the observation $\mathbf{y}$ from an intermediate latent state along the original denoising path and use it for subsequent projection. This design guarantees CS consistency while preventing CSGuard from deviating from the original DMs' generation trajectory, thereby preserving both watermark effectiveness and image quality.

% 贡献
Our main contributions are summarized as follows:
\begin{enumerate}
    \item We propose CSGuard, the first forgery-resistant diffusion model watermarking, achieved by breaking the generation–inversion symmetry through compressed sensing constraint, effectively preventing illegal users from exploiting this symmetry for watermark forgery.
    \item We design a consistency-enforcing and dual-fidelity projection through minimum-perturbation projection and trajectory-intrinsic observation, enforcing compressed sensing consistency while preventing CSGuard from deviating from the original trajectory, thus preserving watermarking effectiveness and image quality.
    \item Experimental results demonstrate that CSGuard achieves significant forgery resistance, reduces the attack success rate from 100.0\% (baselines) to 28.12\% without sacrificing  watermarking effectiveness, achieves 100\% detection rate on benign images.
\end{enumerate}
% 绑定将对称性绑定到A，使满足一致性约束而抗伪造攻击；
% 通过最小扰动投影和内生y构建，保障生成质量和水印能力；
% 实验效果

\section{Related Work}

\textbf{Latent-based DMsMark.}
Latent-based DMsMark embeds watermarks into the initial latent variable of DMs and relies on inversion for verification, enabling train-free deployment with minimal visual distortion. 
Tree-Rings \citep{DBLP:conf/nips/WenKGG23} embed watermarks in the Fourier domain using structured ring patterns, a strategy later adapted for multi-user scenarios in RingID \citep{DBLP:conf/eccv/CiYSS24}. 
WIND \citep{DBLP:conf/iclr/ArabiFWHC25} uses a two-stage framework: Fourier-based group retrieval followed by initial noise matching, enabling efficient, distortion-free watermark verification.
However, such frequency-domain manipulations perturb the latent distribution, deviating from the model’s native Gaussian prior.
In response, a second line of work prioritizes distributional fidelity. GaussShading \citep{DBLP:conf/cvpr/YangZCF0Y24} preserves the latent distribution by sampling watermark-carrying points from equal-probability partitions of the Gaussian density, while PRCMark \citep{DBLP:conf/iclr/GunnZS25} maintains magnitudes and only flips latent signs according to the watermark sequence, both ensuring statistical indistinguishability without altering the generative process. 
More recently, efforts have shifted toward robustness against image perturbations. GaussMarker \citep{DBLP:conf/icml/LiHHH25} introduces dual-domain embedding (symbol and frequency), CoSDA \citep{DBLP:conf/aaai/FangCYC0C25} mitigates internal diffusion drift and external noise via compensatory sampling and alignment, and TraceMark \citep{DBLP:journals/corr/abs-2503-23332} employs binary-guided resampling of high-magnitude latent components to minimize reconstruction error while preserving Gaussianity.

Despite their advances, these methods share a critical dependency: they utilize DMs' inversion for verification, which simultaneously enables adversaries to recover watermarked latent variable with arbitrary models and forge malicious content. This fundamental vulnerability, inherent to the symmetric architecture of DMs, remains unaddressed.

We discuss model-based methods and the concurrent work PAI~\citep{DBLP:journals/corr/abs-2601-06639}, along with detailed comparisons in Appendix~\ref{app:related_work}, as they differ fundamentally in security objectives and deployment assumptions.

\section{Preliminaries}

\subsection{Diffusion Models and Inverse DDIM}
\textbf{Generation.} DMs synthesize images $\mathcal{G}_{T \to 0}(z_T; u_{\theta}, c)$ with the conditioning prompt $c$ by iteratively denoising a random noise sample $z_T \sim \mathcal{N}(0, \mathbf{I})$. 
At each denoising step $t$, the network $u_{\theta}(\cdot)$ predicts $z_{0|t}$, an estimate of the clean latent $z_0$ from the noisy input $z_t$.
And then computes:
$
    z_{t-1} = \sqrt{\alpha_{t-1}}\, z_{0|t} + \sqrt{1 - \alpha_{t-1}}, u_{\theta}(z_t, t),
$
for $t = T, \dots, 1$, where $\alpha_t = \prod_{s=0}^t (1 - \beta_s)$. The final latent $z_0$ is decoded to an image through decoder $x = \mathcal{D}(z_0)$.

\textbf{Inversion.} Given an image $x$, its latent $z_0 = \mathcal{E}(x)$ can be inverted back to the initial variable via inverse DDIM sampling \citep{DBLP:conf/nips/HoJA20}:
$
    z_{t+1} = \sqrt{\alpha_{t+1}}\, z_{0|t} + \sqrt{1 - \alpha_{t+1}}\, u_{\theta}(z_t, t),
$
for $t = 0, \dots, T-1$, yielding $z_T = \mathcal{I}_{0 \to T}(z_0; u_{\theta})$. This bidirectional mapping, grounded in the estimation of $\hat{z}_0$, enables both synthesis and latent recovery.

\subsection{Forgery Attack}
Given a watermarked image $x^{\text{w}}$ generated by a legitimate user under prompt $c$, an adversary can perform a forgery attack (reprompt attack) \citep{DBLP:conf/cvpr/0025LTFQ25} as follows. 
First, the attacker inverts $x^{\text{w}}$ to recover its initial latent variable $\tilde{z}_T = \mathcal{I}_{0 \to T}(\mathcal{E}(x^{\text{w}}); u_{\theta})$ using any off-the-shelf DM. 
Subsequently, the attacker selects a malicious prompt $c^{\text{adv}}$ and synthesizes a forged image:
$
    x^{\text{adv}} = \mathcal{D}\big( \mathcal{G}_{T \to 0}(\tilde{z}_T; u_{\theta}, c^{\text{adv}}) \big).
$
The resulting image $x^{\text{adv}}$ exhibits semantic content aligned with $c^{\text{adv}}$, while preserving the latent signature originally embedded by the legitimate user. 
Consequently, the watermark verification mechanism will erroneously attribute $x^{\text{adv}}$ to the user. 
This attack succeeds without requiring access to the original generation model or knowledge of the watermark method, highlighting a fundamental vulnerability in latent-based watermark schemes.

\subsection{Compression Sensing}
Compressed sensing (CS) \citep{DBLP:journals/tit/Donoho06, DBLP:journals/pami/ChenZLZYZCZ25} is a signal acquisition paradigm that enables direct acquisition of a compressed representation $\mathbf{y} \in \mathbb{R}^M$ of a high-dimensional signal $\mathbf{x} \in \mathbb{R}^N$ via the linear measurement model:
$
    \mathbf{y} = \mathbf{A} \mathbf{x},
$
where $\mathbf{A} \in \mathbb{R}^{M \times N}$ is a random sample matrix, the compression ratio is $\gamma = M/N < 1$.
Recovery of $\mathbf{x}$ from $\mathbf{y}$ is feasible only when the matrix $\mathbf{A}$ is known. Without access to $\mathbf{A}$, the inverse problem $\mathbf{y} \mapsto \mathbf{x}$ is ill-posed and admits infinitely many solutions. Hence, $\mathbf{A}$ serves as a recovery key: only parties possessing $\mathbf{A}$ can recover $\mathbf{x}$ from its compressed observation $\mathbf{y}$.

\section{Methodology}

This section formalizes the design of CSGuard.  
% Guided by the insight that latent-based watermarking’s vulnerability stems from exploitable generation–inversion symmetry,
DMsMark's generation–inversion symmetry renders it vulnerable to forgery attack, therefore, we (1) propose to leverage compressed sensing to bind this symmetry to a secret matrix, thus breaking this symmetry for robust against forgery attack, and (2) propose a consistency-enforcing and dual-fidelity projection to satisfy CS constraint and preserve the watermarking fidelity. 
Together, they enable CSGuard to resist forgery and preserve generation quality and watermarking effectiveness.
All proofs are detailed in ~\ref{app:proof}.

% \begin{figure}[tb]
%   \begin{center}
%     \centerline{\includegraphics[width=\columnwidth]{fig/framework.pdf}}
%     \caption{
%       Framework of our CSGuard.\textbf{false}%公式有错
%     }
%     \label{icml-historical}
%   \end{center}
%   \vskip -0.1in
% \end{figure}

\begin{figure*}[tb]
  \begin{center}
    \centerline{\includegraphics[width=0.9\linewidth]{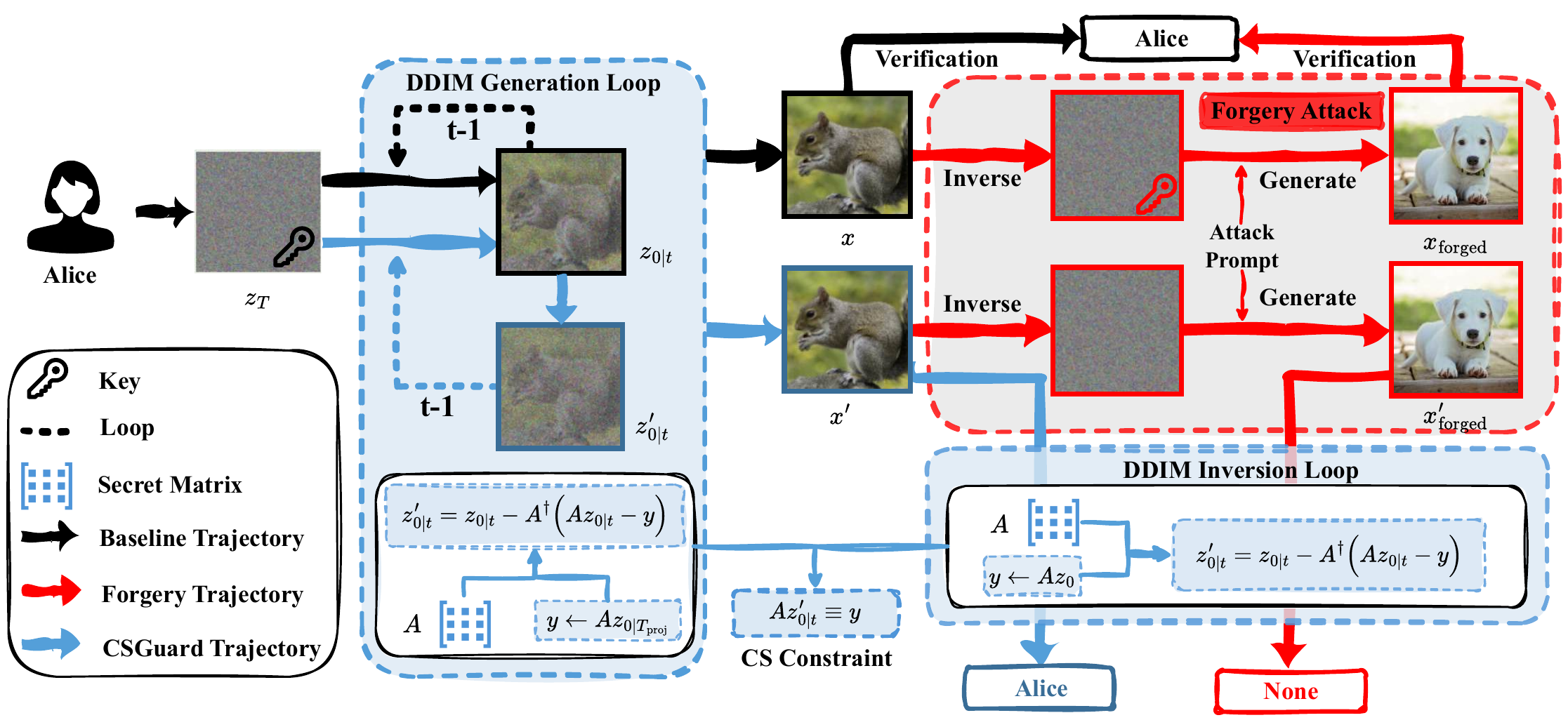}}
    \caption{
      Overview of CSGuard.
      By enforcing CS consistency on intermediate latents, CSGuard binds generation–inversion symmetry to a secret matrix, thereby preventing adversaries from exploiting this symmetry to forge, rendering it robust against forgery attacks.
    }
    \label{fig:csgurad}
  \end{center}
  \vskip -0.4in
\end{figure*}

\subsection{Threat Model}
\textbf{Adversary’s Goals and Capabilities.}
The adversary aims to steal the victim’s watermark from a watermarked image and use it to forge a new image that verifies as containing the victim’s watermark.
This enables attackers to generate malicious content without being traced and even frame innocent parties.
% This enables false attribution, undermining the trustworthiness of the watermarking system. 
The attacker has their own DMs but does not know the victim’s DM or details of DMsMark.

\textbf{Defender's Goal and Capabilities.}
The defender aims to embed user-specific watermarks into generated images without degrading visual fidelity, and to minimize the success rate of forgery attacks, preventing adversaries from forging images that falsely verify under a victim’s watermark.

% \subsection{Overview}
% 还得改

\subsection{Design Principle}
% \subsection{Anti-Forgery through CS Constraint}
The core vulnerability of latent-based DMsMark stems from its reliance on the intrinsic symmetry of DMs. Specifically, the approximate invertibility between generation and inversion. While this symmetry enables reliable watermark extraction for legitimate users, it also permits adversaries to exploit the same mechanism to forge content by reusing initial watermarked latent variables.

% \textit{\uline{
Therefore, the key to ensuring watermark performance and resisting forgery attacks lies in guaranteeing the legitimate users' ability to utilize symmetry while making it inaccessible to illegal users.
In CS, this issue is resolved by design: the reconstruction symmetry between $\mathbf{y}$ and $\mathbf{x}$ is conditional, it only holds for parties who possess the sensing matrix $\mathbf{A}$.
We transpose this principle to DMsMark by binding the generation–inversion symmetry to a secret $\mathbf{A}$. By ensuring the watermark verification is functional for legitimate users while infeasible for adversaries, we can effectively convert a public vulnerability into a private capability.

Specifically, we propose to embed a secret matrix $\mathbf{A}$ into the generative process, constraining the generative images' latent representations satisfy the CS constraint $\mathbf{A} z_{0|t} \equiv \mathbf{y}$ (see Fig~\ref{fig:csgurad}). During verification, only the legitimate user who possesses $\mathbf{A}$ can impose the same constraint in inversion, thereby preserving the generation–inversion symmetry and the watermark’s validity. Unauthorized users who lack $\mathbf{A}$ are unable to exploit the model’s symmetry for forgery.
Moreover, legitimate users are strongly incentivized to safeguard $\mathbf{A}$, as its disclosure would enable adversaries to forge watermarked content under the user’s identity, leading to false attribution and potential liability.

\subsection{Consistency-Enforcing and Dual-Fidelity Projection}

% Inspired by the principle of CS constrained generation-inverse symmetry, 
To enable CSGuard to generate images that are verifiable while defending against forgery attack, we formalize two essential constraints:
(1) \textit{Consistency Constraint:} $\mathbf{A} z_{0|t} \equiv \mathbf{y}$, which binds the generation–inversion symmetry to a secret $\mathbf{A}$; and
(2) \textit{Fidelity Constraint:} Guarantees the quality of the generative image and the effectiveness of the watermark.
To simultaneously satisfy consistency and fidelity constraints, we propose a consistency-enforcing and dual-fidelity projection that combines a minimal-perturbation consistency projection with trajectory-intrinsic observation construction.

\textbf{Minimal-Perturbation Consistency Projection.}
Motivated  by the range-null space decomposition in image restoration \citep{DBLP:conf/aaai/WangHYZ23,DBLP:conf/iclr/WangYZ23,DBLP:journals/pami/ChenZLZYZCZ25}, 
% we introduce a projection operator that satisfies the consistency constraints and allows us to optimize the watermark and perceptual fidelity freely.Specifically, 
we formalize the consistency-enforcing projection as an orthogonal projection onto an affine constraint set. thus ensuring constraint satisfaction with minimal perturbation to the latent state, which is crucial for preserving dual quality.

\begin{definition}[Consistency Constraint Set]
Given a full-rank measurement matrix $\mathbf{A} \in \mathbb{R}^{M\times N}$ and denote its Moore-Penrose pseudoinverse as $\mathbf{A}^\dagger = \mathbf{A}^\top (\mathbf{A} \mathbf{A}^\top)^{-1}$, and an observation $\mathbf{y} \in \mathbb{R}^M$ , the consistency constraint set is the affine subspace 
$\mathcal{M}_{\mathbf{A,y}}:=z_{0|t}-\mathbf{A}^\dagger(\mathbf{A}z_{0|t}-\mathbf{y})$.
\end{definition}

\begin{lemma}[Minimal-Perturbation Consistency Projection]
\label{lemma:lemma1}
Let $\mathbf{A} \in \mathbb{R}^{M\times N}$ with full row rank, and denote its Moore-Penrose pseudoinverse as $\mathbf{A}^\dagger = \mathbf{A}^\top (\mathbf{A} \mathbf{A}^\top)^{-1}$. For any intermediate latent $z_{0|t} \in \mathbb{R}^d$, the projected latent is:
$
     z'_{0|t} = z_{0|t} - \mathbf{A}^\dagger(\mathbf{A} z_{0|t} - \mathbf{y}),
    \label{eq:projection}
$
Then:
    (i) $z'_{0|t}$ satisfies the consistency constraint: $\mathbf{A} z'_{0|t} \equiv \mathbf{y}$.
    (ii) The mapping $z_{0|t} \mapsto z'_{0|t}$ satisfies the minimal perturbation:   
    $z'_{0|t}=\arg \min _{z\in\mathcal{M}_{\mathbf{A,y}}}||z'_{0|t}-z_{0|t}||_2$.
    (iii) $z'_{0|t}$ shares the same null component with $z_{0|t}$.
    % (ii) The mapping $z_{0|t} \mapsto z'_{0|t}$ is the unique orthogonal projection onto the affine subspace $\mathcal{M}_{\mathbf{A}, \mathbf{y}} := \{ z \in \mathbb{R}^d : \mathbf{A} z = \mathbf{y} \}$.
    % (iii) The projection decomposes $z'_{0|t}$ into orthogonal components.
\end{lemma}

%必然的结果
\begin{corollary}[Invariance under Null-Space Optimization]
\label{coro:coro1}
Since the projection operator leaves the original null-space component of $z_{0|t}$ unchanged, any perturbation $\Delta z \in \textit{Null}(\mathbf{A})$ preserves the satisfaction of the constraint $\mathbf{A} z \equiv \mathbf{y}$. In other words, the consistency constraint is invariant to arbitrary modifications within $\textit{Null}(\mathbf{A})$.
\end{corollary}
% (I−A†A)z 就是原始 z 在 null space 中的分量，且该分量在投影后被完整保留。即只修改range分量来满足一致性（z=A*AZ+NULL,z'=A*y+NULL）

This property enables us to enforce CS consistency with minimal perturbation regardless of how $z_{0|t}$ is sampled or the embedded watermark.
Furthermore, the full flexibility in the null space allows DMsMark to iteratively optimize its null-space component, thus preserving both generative capability and watermark effectiveness~\citep{DBLP:conf/iclr/WangYZ23}.

\textbf{Trajectory-Intrinsic Observation Construction.}
The minimal-perturbation projection enables CSGuard to satisfy the CS consistency constraint with minimal perturbation. However, satisfying the fidelity constraint requires more than mere adherence to consistency; crucially, the projected latent variable $z'_{0|t}$ must remain in close proximity to its original generative trajectory, thereby preserving the semantic coherence and watermark integrity of the generated image.

We identify that naive observations, such as a  random matrix, severely degrade the generation quality and propose constructing $\mathbf{y}$ from an intermediate latent state along the original denoising path to preserve dual fidelity:

% To keep the projected latent on the generative trajectory, we construct $\mathbf{y}$ from an intermediate latent state along the original denoising path via a trajectory-intrinsic observation.

\begin{proposition}[Random Observation Destroys Generation Fidelity]
\label{proposition:proposition1}
    Let $\mathbf{y} \sim \mathcal{N}(0, I)$ be a randomly sampled observation. Then the projected latent significantly deviates from the original trajectory and destroys the generation fidelity. 
\end{proposition}
\begin{definition}[Trajectory-Intrinsic Observation]
Given a full-rank measurement matrix $\mathbf{A} \in \mathbb{R}^{m \times d}$, the trajectory-intrinsic observation is constructed as $\mathbf{y}=\mathbf{A}z_{0|T_{proj}}$, where $z_{0|T_{proj}}$ is a intermediate latent at step $T_{proj}$.
\end{definition}
We formalize the key properties this observation as follows:
\begin{lemma}[Johnson--Lindenstrauss Property]
\label{Lemma:Lemma2}
Let $\mathcal{X}=\{\mathbf{x}_1,\ldots,\mathbf{x}_N\}\subset\mathbb{R}^d$ be a finite set and let $\varepsilon\in(0,1)$.
There exists a random linear map $\mathbf{A}\in\mathbb{R}^{m\times d}$ with
$m = O(\varepsilon^{-2}\log N)$ such that, with probability at least $1-\delta$,
for all $i,j\in[N]$:
$
(1-\varepsilon)\|\mathbf{x}_i-\mathbf{x}_j\|_2
\;\le\;
\|\mathbf{A}(\mathbf{x}_i-\mathbf{x}_j)\|_2
\;\le\;
(1+\varepsilon)\|\mathbf{x}_i-\mathbf{x}_j\|_2.
$
Equivalently, $\mathbf{A}$ approximately preserves pairwise distances over $\mathcal{X}$.
\end{lemma}
\begin{proposition}[Trajectory-Intrinsic Observation Preserves Fidelity]
\label{proposition:proposition2}
    Let $\mathbf{y}$ be a trajectory-intrinsic observation and $\mathbf{A}$ satisfy the JL property, the projected latent remains close to the original denoising trajectory and preserves the fidelity.
\end{proposition}
% Based on the above theory, we design CSGuard to simultaneously satisfy both consistency and fidelity constraints:
Guided by this theory, CSGuard is designed to jointly satisfy consistency and fidelity constraints:
\begin{theorem}[CSGuard Simultaneously Satisfies Consistency and Fidelity Constraint]
\label{theorem:theorem1}
Let $\mathbf{A} \in \mathbb{R}^{M\times N}$ be a full-rank measurement matrix satisfying the JL property. Let $\mathbf{y}=\mathbf{A}z_{0|T_{proj}}$ be the trajectory-intrinsic observation, where $z_{0|T_{proj}}$ lies on the original denoising trajectory of the DMs at step $T_{proj}$.
For any $t \in [0,T_{proj}]$, CSGuard projects the latent $z_{0|t}$ as:
$z'_{0|t} = z_{0|t} - \mathbf{A}^\dagger(\mathbf{A} z_{0|t} - \mathbf{y})$. 
Then, the following key properties hold:
% \begin{enumerate}
%     \item Consistency Constraint: the projected latent $z'_{0|t}$ the minimal-perturbation solution to the CS constraint.
%     \item Fidelity Constraint: (i) the mapping $z_{0|t} \mapsto z'_{0|t}$ satisfies the minimal perturbation. (ii) The $z'_{0|t}$ remains close to the original generative trajectory and preserves the semantic coherence and watermark integrity.
% \end{enumerate}
 (1) Consistency Constraint: the projected latent $z'_{0|t}$ the minimal-perturbation solution to the CS constraint.
(2) Fidelity Constraint: (i) the mapping $z_{0|t} \mapsto z'_{0|t}$ satisfies the minimal perturbation. (ii) The $z'_{0|t}$ remains close to the original generative trajectory and preserves the semantic coherence and watermark integrity.
\end{theorem}

By combining a minimum-perturbation consistency projection with trajectory-intrinsic observation construction, CSGuard can satisfy the CS constraint without modifying the original generative trajectory, preserve both its image generation capability and watermark functionality while resisting forgery attacks.
% All proofs are detailed in ~\ref{app:proof}.
% or introducing additional perturbation like random observation (proposition~\ref{proposition:proposition1}), 
% perform iterative denoising to 
% preserve both its image generation capability and watermark functionality 
% (lemma~\ref{lemma:lemma1}, proposition~\ref{proposition:proposition2}), while resisting forgery attacks(theorem~\ref{theorem:theorem1}).
All proofs and the securit analysis are detailed in ~\ref{app:proof} and~\ref{app:security_analysis}, respectively.

% 这段落icml有中文解释

\subsection{Watermarked Image Generation}

% 补充开头，使用gen公式 $\mathcal{G}_{T \to 0}(z_T; u, c)$ 

% 需要说明的内容：用户拥有A和水印信息，y如何构建，投影比例和压缩比例
% 描述性说法+具体算法细节和流程

% 点名迭代优化null

The core of CSGuard to generate the watermarked image is enforcing that the intermediate latent $z_{0|t}$ satisfies the unique constraint of the legitimate user. 
For the legitimate user Alice, CSGuard assigns a unique private matrix $\mathbf{A} \in \mathbb{R}^{M\times N}$ and a watermark key $k$ for generating watermarked images, denoted as $\mathcal{G}_{T \to 0}(z_T; u_{\theta}, c, \mathbf{A},k)$.

Our generation procedure integrates a projection operator into the DDIM denoising loop. 
Starting from sampling an initial latent variable that carries the watermark information, we utilize the schema from Gaussian Shading \citep{DBLP:conf/cvpr/YangZCF0Y24} to sample $z_{T}$ with Alice's key $k$.
When $t > T_{proj}$, we follow the standard generation process to predict noise and denoise.
During the $T_{proj}$ step, we construct the observation with the secret matrix and the latent representation of the current estimated clean images $\mathbf{y} \leftarrow \mathbf{A}z_{0|T_{proj}}$.
This $\mathbf{y}$ defines the constraint subspace $\{z_{0|t}: \mathbf{A}z_{0|t} \equiv \mathbf{y}\}$.
When $t < T_{\text{proj}}$, we project $z_{0|t}$ onto the constraint subspace through Eq.~\eqref{eq:projection} to guarantee the consistency constraint and update $z_{0|t}$ and use it for the subsequent denoising step operates, thereby ensuring the generation process satisfy constraint that blind with matrix $\mathbf{A}$.
Finally, the decoding step  produces the watermarked image, which adheres to both constraint consistency and fidelity. 
This procedure ensures that the generated image satisfies CS constraint, while preserving both the DM and watermark’s native performance in the null space of $\mathbf{A}$.
The complete algorithm is given in Algorithm~\ref{alg:generation} in~\ref{app:alg}.

\subsection{Watermark Verfication}

The core of CSguard’s verification stage is to enforce the same constraint as in generation, using the secret matrix $\mathbf{A}$ to ensure that intermediate latents $z_{0|t}$ satisfy $\mathbf{A} z_{0|t} \equiv \mathbf{y}$, thereby recovering the initial latent variable carrying the watermark, denoted as $\mathcal{I}_{0 \to T}(z_0; u_{\theta},\mathbf{A}, k)$.

Unlike generation, where the observation $\mathbf{y}$ is derived from an intermediate latent at step $T_{\text{proj}}$, here $\mathbf{y}$ is constructed directly from the final latent $z_0$ of the input image:
$\mathbf{y} \leftarrow \mathbf{A} z_0, \text{where } z_0 = \mathcal{E}(x)$.
This design is motivated by two primary considerations: 
(1) For authentic watermarked images, $z_0$ inherently lies within the affine subspace defined by $\mathbf{A}$ during generation. Using it to reconstruct $\mathbf{y}$ ensures that the subsequent constrained inversion remains within this subspace.
(2) The latent $z_0$ shares semantic and structural coherence with the intermediate estimates $z_{0|t}$ during generation. As such, it serves as a natural reference point for stabilizing the recovery trajectory under the constraint.

The verification procedure begins by encoding $x$ into its latent representation $z_0 = \mathcal{E}(x)$. The core of the algorithm lies in the inverse DDIM sampling loop, which recovers the initial noise $z_T$ from $z_0$. Crucially, for steps $t < T_{\text{proj}}$, we project the estimated clean latent $z_{0|t}$ onto the constraint subspace $\{\mathbf{z}: \mathbf{A} \mathbf{z} \equiv  \mathbf{y}\}$, thereby ensuring that the recovered trajectory remains aligned with the watermark constraint throughout inversion.
The final step extracts the watermark message from the recovered $z_T$ using the Gaussian Shading\citep{DBLP:conf/cvpr/YangZCF0Y24}, and performs  verification.
This design ensures that only users who possess the secret matrix $\mathbf{A}$ can exploit the generation–inversion symmetry to recover the initial latent variable, effectively preventing unauthorized users from performing forgery attacks or falsely attributing malicious content.
The complete algorithm is given in Algorithm~\ref{alg:verification} in~\ref{app:alg}.

\section{Experiments}
We conduct extensive experiments to demonstrate that CSGuard simultaneously achieves high watermark effectiveness and generation quality, effectively resists forgery attacks, and maintains robustness under diverse image distortions. Furthermore, we perform ablation studies to validate the effectiveness of our trajectory-intrinsic observation design and analyze the impact of compression ratio, projection ratio, and sampling steps. Additional evaluation on different DMs and FPR, and additional discussion, as well as the details of settings, are provided in Appendix~\ref{app:experiment}.

\subsection{Effectiveness of Preserving Dual Fidelity}

\begin{table*}[t]
\centering
\caption{Comparison of DMsMark methods. 
We report bit accuracy (BitAcc), true positive rate (TPR), and CLIP score on benign images, and BitAcc, attack success rate (ASR) on forged images. The subscript reports the absolute drop from performance on benign images.}
\label{tab:wm_result}
\begin{small}
\begin{sc}
\scalebox{0.95}{
\begin{tabular}{l|ccc|cc}
\toprule
\multicolumn{1}{c}{} &\multicolumn{3}{|c|}{Benign Image}  &\multicolumn{2}{c}{Forged image}   \\ \midrule
{Method} & {BitAcc (\%)$\uparrow$} & {TPR (\%)$\uparrow$} & {CLIP Score$\uparrow$} & {BitAcc (\%)$\downarrow$} & {ASR (\%)$\downarrow$}\\
\midrule
Tree-Rings    & -          & 100.0          & 0.3194 & -          & 90.63 \\
RingID    & -          & 96.88          & 0.3047 & -          & 96.88 \\
WIND    & -          & 100.0          & 0.3087   & -          & 100.0 \\
GaussShading  & 100.0      & 100.0          & 0.3160 & $88.79_{11.21}$       & 100.0\\
PRCMark       & 100.0      & 100.0          & 0.3129 & $100.0_{0.00}$      & 100.0\\
GaussMarker   & 99.99       & 100.0          & 0.3078 & $96.43_{3.56}$       & 100.0 \\
{\cellcolor{gray!20}}CSGuard      &{\cellcolor{gray!20}}96.15 &{\cellcolor{gray!20}}100.0 &{\cellcolor{gray!20}}0.3135 &{\cellcolor{gray!20}}$70.26_{25.89}$ &{\cellcolor{gray!20}} 28.12\\
\bottomrule
\end{tabular}}
\end{sc}
\end{small}
\vskip -0.1in
\end{table*}

\begin{figure*}[tb]
  \begin{center}
    \centerline{\includegraphics[width=0.85\linewidth]{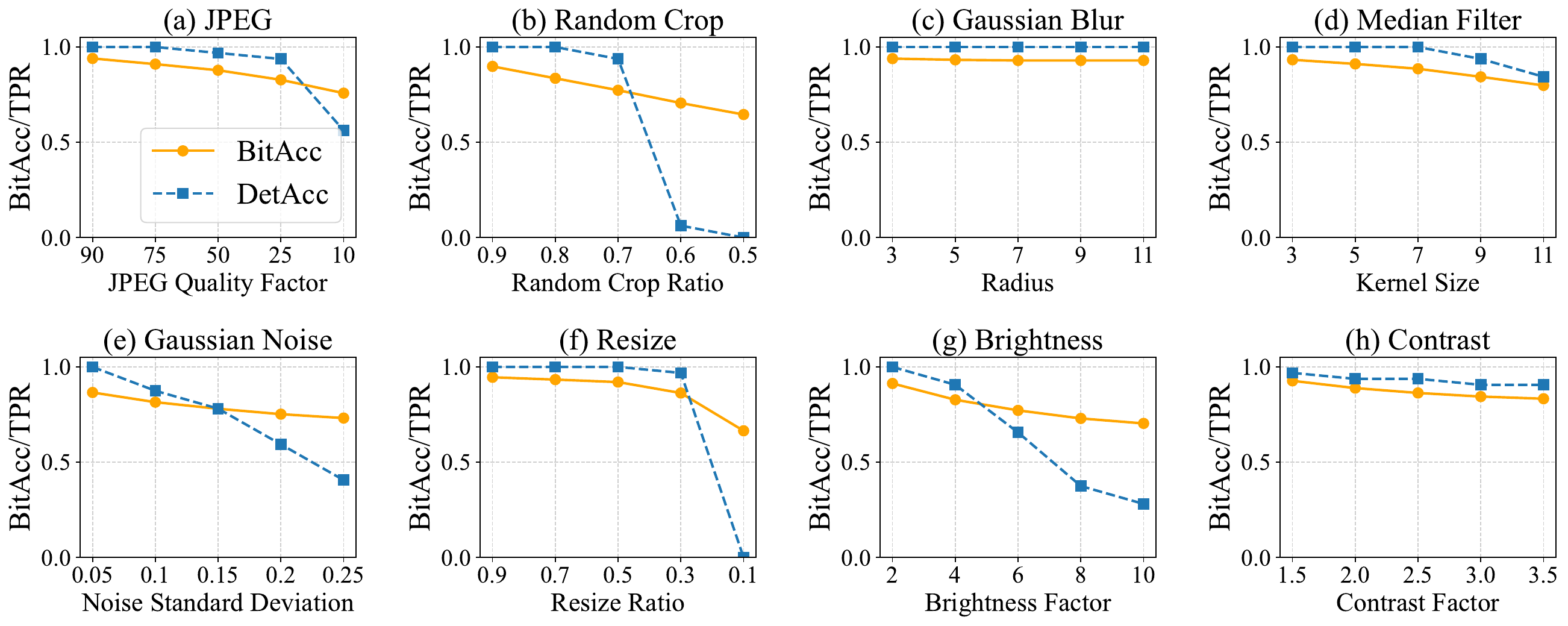}}
    \caption{Watermark performance under varying image distortions of different intensities.}
    \label{fig:attack}
  \end{center}
  \vskip -0.4in
\end{figure*}

We evaluate the watermark effectiveness of our method on benign images against six baselines, including metrics for watermark traceability and generative image quality.

Our method simultaneously achieves reliable watermark detection and high-quality image generation. Specifically, as shown in Table~\ref{tab:wm_result}, CSGuard attains a 100.0\% TPR on benign images, matching all baselines while maintaining a bit accuracy of 96.15\%, slightly below the ideal 100\% but sufficient for robust decoding. Crucially, it preserves generation quality with a CLIP score of 0.3135, comparable to state-of-the-art (SOTA) methods. 
These results demonstrate that CSGuard successfully enables reliable watermark detection without compromising image quality.
Beyond CLIP Score, we provide comprehensive evaluations in Appendix~\ref{app:fidelity}, including FID analysis, stylistic diversity preservation across five prompt categories, and extensive visualizations. These results confirm negligible distributional shift and artifact-free generation, further validating CSGuard's dual-fidelity design.
The generalizability of CSGuard across diverse DMs is detailed in ~\ref{app:abaltion}. 

This is enabled by our consistency-enforcing and dual-fidelity projection design. The minimal-perturbation consistency projection ensures CSGuard’s manipulation of intermediate latents remains minimally perturbative, while the trajectory-intrinsic observation keeps CSGuard’s trajectory closely aligned with the original trajectory, jointly preserving generation quality and watermark effectiveness. 

\subsection{Effectiveness against Forgery Attack}

We evaluate the effectiveness of CSGuard against forgery attack on forged images, demonstrate that CSGuard effectively resists such attacks even when the adversary has full knowledge of our method.

\textbf{Resistance to Standard Forgery.}
The proposed method exhibits unprecedented resistance to forgery, validating the effectiveness of binding the generation-inversion symmetry to the secret matrix through our CS constraints.
As shown in Table~\ref{tab:wm_result}, the bit accuracy of CSGuard on forged images drops to 70.26\%, substantially lower than its benign-image performance of 96.15\%. This clear separation enables reliable discrimination between forged and authentic images through a simple thresholding rule.
Concurrently, the attack success rate (ASR) falls to 28.12\%, a stark contrast to the ASR achieved by baselines, which fail to retain any detectable watermark under the same attack.
This demonstrates that CSGuard not only survives realistic forgery but also provides a verifiable signal of authenticity via measurable degradation in bit accuracy, enabling practical detection of forged content.

While our primary focus is on resisting reprompt-based forgery attacks, CSGuard inherently demonstrates strong resilience against intentional watermark removal and other forgery threats. More evaluations under these attacks and varying attack-target configurations are provided in ~\ref{app:add_attack}.

\textbf{Security under Full-Knowledge Threats.} 
We evaluate CSGuard under a strong threat model where the adversary has full knowledge of our watermarking mechanism and and attempts forgery even with partial information about the secret matrix $\mathbf{A}$.
We systematically evaluate forgery attacks using estimated matrices with similarity to the true secret matrix ranging from 0 (equivalent to brute-force) to 0.96. Experimental results demonstrate that across all similarity levels, forged images maintain bit accuracy between 69.85\% and 72.36\%, and ASR remains in the range of 25.0\%-37.5\%, nearly identical to the uninformed baseline. Notably, even at a high similarity of 0.96, the ASR reaches only 37.5\%, demonstrating that even with complete knowledge of our method and partial information about $\mathbf{A}$, the attacker still cannot reliably forge, further confirming CSGuard's security.

This resilience stems directly from our CS constraints, which tightly couple the generation--inversion symmetry to the exact secret matrix. Without precise knowledge of $\mathbf{A}$, the adversary cannot enforce the same constraint during inversion as during generation, rendering both brute-force and approximate-matrix forgery attacks fundamentally ineffective.

\subsection{Effectiveness against Post-Processing}

To comprehensively assess the robustness of CSGuard, we evaluate its performance under eight representative image distortion operations with varying intensity parameters, as summarized in Fig~\ref{fig:attack}.
% For each distortion type, we systematically vary the intensity parameter.

CSGuard exhibits strong robustness under a wide range of image distortions. Specifically, the bit accuracy and TPR remain consistently high, typically under JPEG compression, Gaussian blur, median filtering, resizing, and contrast adjustment. Under Gaussian blur, TPR remains at 100\% across all settings, while bit accuracy stays above 92\%. Under median filtering and contrast adjustment, bit accuracy exceeds 75\%, and TPR consistently exceeds 80\%.
For random cropping, Gaussian noise, and brightness adjustment, performance remains robust at mild-to-moderate intensities, with noticeable degradation only under extreme perturbations. 
This confirms that CSGuard preserves both watermark integrity and detectability under common post-processing operations.

\subsection{Ablation Studies and Analysis}
\begin{table*}[tb]
\centering
\caption{Watermarking performance under varying sampling step. 
We report BitAcc, TPR, CLIP score, and ASR of CSGuard, along with its time overhead ratio. The time ratio is computed as the execution time with projection divided by that without projection.
}
\label{tab:wm_step}
\begin{small}
\begin{sc}
\scalebox{0.875}{
\begin{tabular}{l|c|cccc|cc}
\toprule
\multicolumn{2}{c}{} &\multicolumn{4}{|c|}{Benign Image}  &\multicolumn{2}{c}{Forged image} \\ \midrule
Method &{Step} & {BitAcc (\%)$\uparrow$} & {TPR (\%)$\uparrow$} & {CLIP Score$\uparrow$} &{TimeRatio $\downarrow$}& {BitAcc (\%)$\downarrow$} & {ASR (\%)$\downarrow$}\\
\midrule
\multirow{4}{*}{CSGuard} &10   & 94.81          & 100.0          & 0.2896 &1.004$\times$ & $68.45_{26.36}$          & 15.62 \\
 &25 & 95.28      & 100.0      & 0.3104 &1.005$\times$ & $68.87_{26.41}$       & 15.63\\
 &50 & 96.15      & 100.0      & 0.3135 &1.005$\times$ &$70.26_{25.89}$       &28.12\\
 &100 & 95.23     & 100.0    & 0.3179 &1.007$\times$ & $68.21_{27.02}$       & 9.36 \\
\bottomrule
\end{tabular}}
\end{sc}
\end{small}
\vskip -0.3in
\end{table*}

% step和不同模型影响，不同阈值/FPR设置
% 2. 不同A选择柱状图，bit，det，+质量指标。尝试不同A，即公开算子在同proj情况下的情况，同时由于公开，说明这些方法是缺乏安全性的，2段（setting+结果）（来不及先不放了）

\begin{wrapfigure}{r}{0.5\textwidth}
  \vskip -0.3in
\centering
  \includegraphics[width=0.48\textwidth]{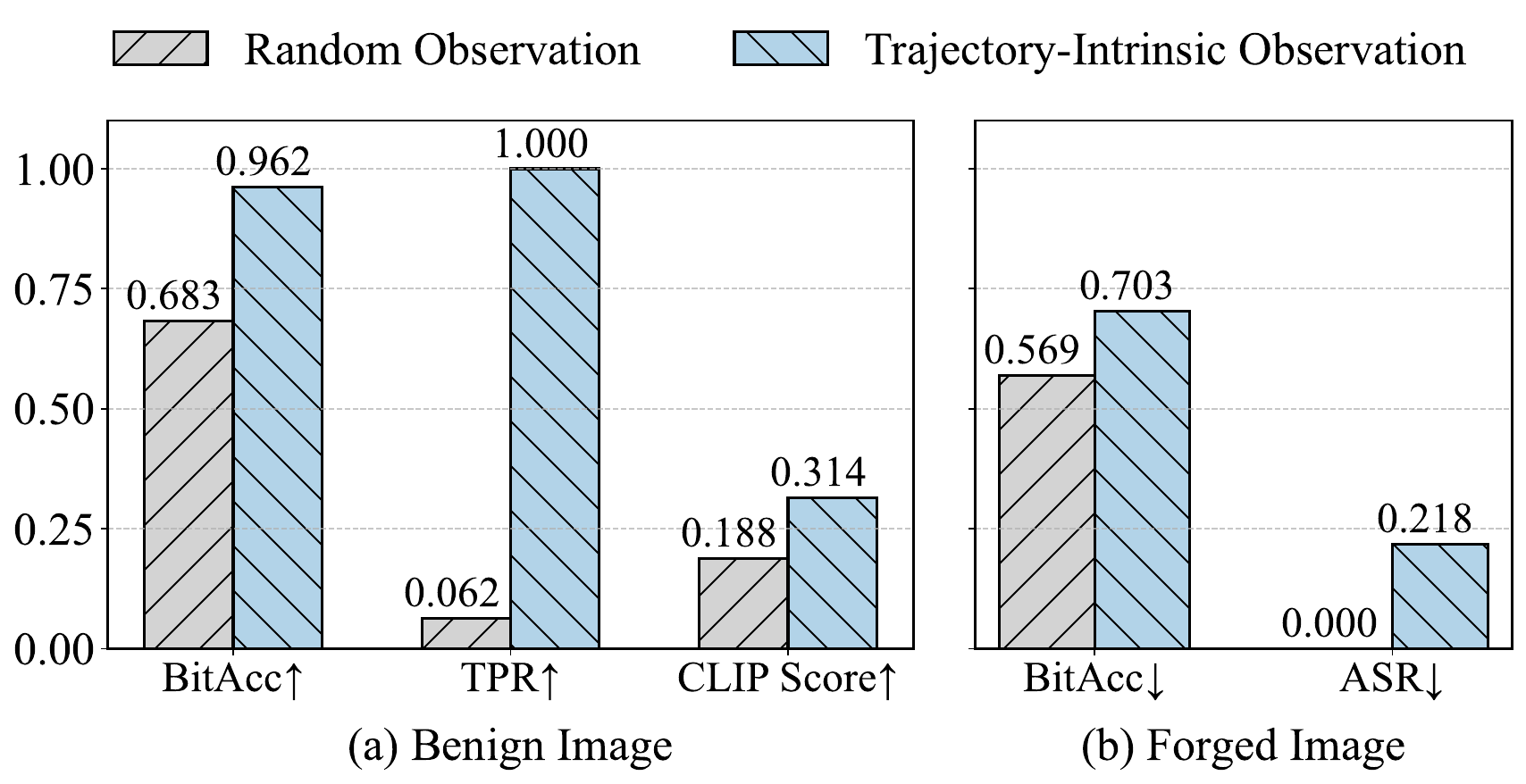}
  \caption{Ablation study on the impact of observation construction.
  }
  % \Description{}
  \label{fig:y_comparison}
  % \vskip -0.2in
\end{wrapfigure}

\textbf{Ablation Study on Observation.}
We conduct an ablation study to validate the efficacy of our trajectory-intrinsic observation construction, and underscore its essential role in both watermark effectiveness and generative quality.
As shown in Fig~\ref{fig:y_comparison}, the random observation leads to a catastrophic degradation across all metrics.
Specifically, although the ASR is near zero, this does not signify robustness; rather, it stems from a complete breakdown of watermark functionality, as evidenced by the TPR of only 0.062 on benign images. Concurrently, the CLIP score drops significantly from 0.314 to 0.188, indicating that the random observation disrupts the generative process itself, which aligns with Lemma~\ref{lemma:lemma1}.

\textbf{Analysis on Compressed Ratio.}
We analyze the impact of the compression ratio ($\text{cs\_ratio} \in \{0.1, 0.2, \dots, 0.9\}$ ), which governs the dimensionality of the matrix $\mathbf{A}$. As shown in Fig.~\ref{fig:proj_and_cs}(a), both bit accuracy on benign and forged images decline gradually with increasing $\text{cs\_ratio}$.
Concurrently, the ASR decreases monotonically, indicating stronger forgery resistance at higher dimensionality of $\mathbf{A}$. Crucially, the TPR on benign images remains near-perfect and the CLIP score stays stable, confirming that generative fidelity is preserved. 
It's worth noting that our default setting $\text{cs\_ratio}=0.8$ achieves a favorable trade-off, achieves robust forgery resistance while maintaining high bit accuracy on benign images and undiminished generation quality.

\begin{wrapfigure}{r}{0.55\textwidth}
\vskip -0.1in
\centering
  \includegraphics[width=0.54\textwidth]{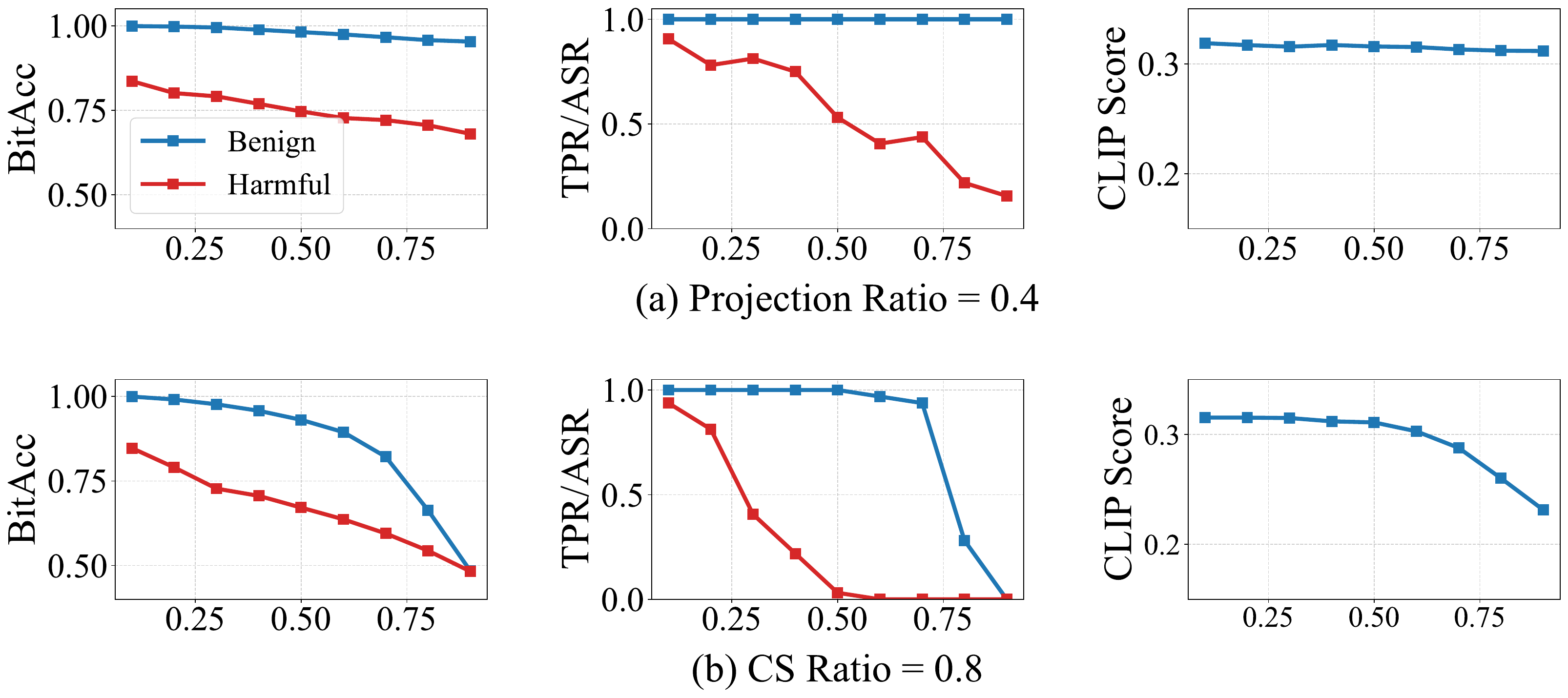}
  \caption{Watermark performance under varying CS recovery ratios and projection ratios. Top: Projection ratio fixed at 0.4, CS ratio varied. Bottom: CS ratio fixed at 0.8, projection ratio varied.
  }
  % \Description{}
  \label{fig:proj_and_cs}
  \vskip -0.1in
\end{wrapfigure}

\textbf{Analysis on Projection Ratio.}
We analyze the effect of the projection ratio  ($\text{proj\_ratio}$), which governs the $T_{proj}$, with results shown in Fig.~\ref{fig:proj_and_cs}(b). As the $\text{proj\_ratio}$ increases from 0.1 to 0.9, the ASR drops sharply from 0.94 to near zero, indicating stronger resistance to forgery. However, this gain comes at a cost: BitAcc on benign images declines from 0.99 to 0.5, the TPR falls from 1.0 to 0, and CLIP similarity decreases from 0.31 to 0.23. This demonstrates a fundamental trade-off between dual fidelity and resistance to forgery.
Our default setting $\text{proj\_ratio}=0.4$ strikes an optimal balance, maintains high bit accuracy, near-perfect TPR, and stable CLIP score, while suppressing ASR.

\textbf{Analysis on Sampling Step.}
We evaluate CSGuard under varying diffusion sampling steps, with results summarized in Table~\ref{tab:wm_step}. Across all settings, the method maintains consistently high watermark effectiveness on benign images: bit accuracy exceeds 94.81\%, TPR remains 100.0\%, and CLIP similarity stays stable around 0.31, confirming minimal impact on generation quality and watermark performance. Crucially, a clear separation persists between bit accuracy of benign and forged images and ASR is suppressed below 28.12\% in all cases, and drops to 9.36\% at step = 100. The slightly elevated ASR at step = 50 arises because the adversary uses the same step count during forgery, yet even under this informed setting, CSGuard retains strong resistance. This confirms CSGuard's robust performance across varying sampling steps while consistently thwarting forgery attacks.

\textbf{Analysis on computational overhead.} 
We evaluate the computational overhead introduced by CSGuard's CS projection step. Notably, the time overhead ratio remains near unity across varying sampling steps, ranging from 1.004$\times$ to 1.007$\times$ compared to the baseline diffusion process without CS Constraint. This confirms that the CS constraint introduces negligible computational overhead (less than 1\%), making it practical for real-world deployment scenarios.

\vskip -0.2in
\section{Conclusion and Limitation}
We present CSGuard, the first DMsMark scheme that resists forgery attacks and preserves both watermark and image quality. By binding the generation–inversion symmetry to a secret sensing matrix via a CS constraint, CSGuard restricts watermark embedding and verification to authorized users, eliminating the structural vulnerability exploited by forgery. Grounded in minimal-perturbation projection and trajectory-intrinsic observation, our dual-fidelity design ensures alignment with the original diffusion trajectory, preserving both watermarking efficacy and image quality. Experiments show CSGuard reduces forgery ASR from 100.0\% to 28.12\%, while maintaining 100\% TPR and competitive CLIP scores, establishing a new paradigm for secure, training-free DMsMark.

\textbf{Limitation.}
CSGuard is vulnerable to model-identical attacks (see~\ref{app:abaltion}). Although our distance-based detection mechanism mitigates this risk, it serves as a post-hoc remedy rather than a fundamental solution. Resolving this vulnerability remains a critical  future work. Additionally, while CSGuard achieves high fidelity on benign images, its robustness under aggressive post-processing and its bit accuracy still leave room for improvement. We will address these aspects in subsequent work.

\bibliography{ref}
\bibliographystyle{abbrvnat}

%%%%%%%%%%%%%%%%%%%%%%%%%%%%%%%%%%%%%%%%%%%%%%%%%%%%%%%%%%%%

\appendix

\section{Appendix}

This appendix provides supplementary materials that support the main claims of CSGuard. The content is organized as follows:

\begin{itemize}
    \item Appendix A.1 (Extended Related Work and Comparisons): We discuss model-based watermarking methods and provide a detailed comparison with the concurrent work PAI~\citep{DBLP:journals/corr/abs-2601-06639}, clarifying the distinct security objectives and mechanism designs that motivate our separate treatment in the main text.
    
    \item Appendix A.2 (Details of Proofs Sketch): We present complete proof sketches for all theoretical results stated in Section 4, including Lemma 4.2 (Minimal-Perturbation Consistency Projection), Proposition 4.4/4.7 (Observation Construction), and Theorem 4.8 (Dual-Constraint Satisfaction). The full algorithms for watermarked generation (Algorithm 1) and verification (Algorithm 2) are also provided.
    
    \item Appendix A.3 (Security Analysis of the Secret Matrix): We formally analyze the security of the secret projection matrix $\mathbf{A}$, including key generation guarantees, space complexity, and resistance against brute-force attacks.
    
    \item \textbf{Appendix A.4 (Details of Experiments)}: 
    \begin{itemize}
        \item A.4.1: Comprehensive experimental settings, including metrics, datasets, and baseline implementations.
        \item A.4.2: Extended fidelity evaluation with quantitative metrics and qualitative visualizations across five stylistic categories.
        \item A.4.3: Additional robustness evaluation against watermark removal attacks and learning-based forgery attack.
        \item A.4.4: Extended analysis including FPR sensitivity, cross-model generalization, attack-target configuration studies, and ablation on alternative projection matrices.
    \end{itemize}
    
    \item Appendix A.5 (Additional Discussion): We discuss (1) the impact of CS constraints on latent Gaussianity and empirical preservation of generative characteristics, (2) the plug-and-play potential of CSGuard and (3) the impact of our work.
\end{itemize}

\subsection{Extended Related Work and Comparisons}
\label{app:related_work}

\subsubsection{Model-based DMsMark}
In contrast to latent-based methods, model-based DMsMark embeds watermarks directly into the model’s parameters typically through fine-tuning or architectural modification, thereby binding the watermark to the model itself rather than its outputs. 
Early approaches such as ArtiFP \citep{DBLP:conf/iccv/YuSAF21} and ProMark \citep{DBLP:conf/cvpr/AsnaniCB0A24} embed watermarks into training data via steganographic encoding, requiring full retraining of the model to transfer the watermark from data to parameters. 
Later works focus on embedding watermarks into the model’s architecture. Stable Signature\citep{DBLP:conf/iccv/FernandezCJDF23}, WOUAF \citep{DBLP:conf/cvpr/Kim0PCY24}, and OmniMark \citep{DBLP:conf/aaai/FeiDXHZ25} modify the VAE decoder to couple watermark information with image generation, enabling user-specific watermarking with a single training pass. Similarly, FlexSecWmark \citep{DBLP:conf/mm/Xiong0F023} introduces a learnable control matrix that binds watermark presence to output quality, preventing unauthorized use. 
More recently, techniques targeting the UNet backbone, such as AquaLoRA \citep{DBLP:conf/icml/FengZH0WL00Y24} and SleeperMark \citep{DBLP:conf/cvpr/WangGZ0H0T25}, they leverage low-rank adaptation (LoRA) to embed watermarks into the core diffusion network. 

Critically, while model-based methods offer stronger binding between watermark and DMs, 
they trade deployment flexibility for security, require retraining or architectural changes, introducing additional computational overhead.
% \textbf{Further discussions and comparisons of related work appear in the appendix.}

\subsubsection{Comparison with PAI.}
Beyond the related works discussed in Section 2, PAI~\citep{DBLP:journals/corr/abs-2601-06639} is the most closely related work. It performs trajectory-level interventions in diffusion sampling to enable forensic watermarking and verification under attacks. However, CSGuard differs fundamentally from PAI in both security objectives and mechanism design.

Regarding security objectives, PAI focuses on post-hoc forensic analysis, aiming to detect whether a verified image has undergone adversarial manipulation. In contrast, CSGuard targets fundamental security by breaking the generation--inversion symmetry to prevent forgery attempts at the source. Rather than merely identifying manipulated images after the fact, CSGuard restricts watermark embedding and verification to authorized users, effectively eliminating the structural vulnerability exploited by forgery attacks.

In terms of mechanism design, PAI employs a key-conditioned deflection mechanism that modifies the predicted noise via a scaling function, introducing statistical bias into the denoising trajectory for hypothesis testing. In contrast, CSGuard leverages compressed sensing to enforce an affine constraint via projection. This binds the generation--inversion symmetry to a secret matrix, rendering unauthorized inversion mathematically ill-posed. This design ensures that only users possessing the secret matrix can correctly embed or verify the watermark, effectively preventing unauthorized users from performing forgery attacks.

Due to these distinct security objectives, we did not include a direct quantitative comparison in the main text. However, our method can be straightforwardly extended to the forensic scenario of PAI to detect image forgery. We implemented PAI's vanilla verification protocol on CSGuard-generated images. Experimental results demonstrate that CSGuard not only prevents forgery (reducing the Attack Success Rate to 28.12\%) but also maintains sufficient signal integrity to achieve 100\% accuracy (Table~\ref{tab:cross_model_attack}) in distinguishing forged from benign images under this evaluation protocol. This confirms that CSGuard provides both preventive security and reliable post-hoc detectability.

We alse note that compared to PAI, our method has certain limitations. PAI excels at semantic-level tamper localization (e.g., face swaps, inpainting), whereas CSGuard currently focuses on global ownership verification and forgery prevention. Additionally, while PAI demonstrates strong adaptability across diverse degradations, CSGuard's performance may experience slight trade-offs under aggressive post-processing or when adversaries use the same diffusion backbone.

\subsection{Details of Proofs Sketch}
\label{app:proof}

\begin{lemma}[Minimal-Perturbation Consistency Projection]
\label{lemma:lemma1}
Let $\mathbf{A} \in \mathbb{R}^{M\times N}$ with full row rank, and denote its Moore-Penrose pseudoinverse as $\mathbf{A}^\dagger = \mathbf{A}^\top (\mathbf{A} \mathbf{A}^\top)^{-1}$. For any intermediate latent $z_{0|t} \in \mathbb{R}^d$, the projected latent is:
\begin{equation}
     z'_{0|t} = z_{0|t} - \mathbf{A}^\dagger(\mathbf{A} z_{0|t} - \mathbf{y}),
    \label{eq:projection}
\end{equation} 
Then:
    (i) $z'_{0|t}$ satisfies the consistency constraint: $\mathbf{A} z'_{0|t} \equiv \mathbf{y}$.
    (ii) The mapping $z_{0|t} \mapsto z'_{0|t}$ satisfies the minimal perturbation:   
    $z'_{0|t}=\arg \min _{z\in\mathcal{M}_{\mathbf{A,y}}}||z'_{0|t}-z_{0|t}||_2$.
    (iii) $z'_{0|t}$ shares the same null component with $z_{0|t}$.
    % (ii) The mapping $z_{0|t} \mapsto z'_{0|t}$ is the unique orthogonal projection onto the affine subspace $\mathcal{M}_{\mathbf{A}, \mathbf{y}} := \{ z \in \mathbb{R}^d : \mathbf{A} z = \mathbf{y} \}$.
    % (iii) The projection decomposes $z'_{0|t}$ into orthogonal components.
\end{lemma}

\begin{proof}[Proof Sketch]
The projection operation can be equivalently rewritten as 
$z'_{0|t}:= \mathbf{A}^\dagger \mathbf{y} + (\mathbf{I} - \mathbf{A}^\dagger \mathbf{A})z_{0|t}$,
where item $\mathbf{A}^\dagger \mathbf{y}$ lies in the range space $\textit{Range}(\mathbf{A})$ of $\mathbf{A}$ the later lies in the null space $\textit{Null}(\mathbf{A})$  of $\mathbf{A}$, and $\textit{Range}(\mathbf{A}) \perp \textit{Null}(\mathbf{A})$.
(i) Direction computation:
$\mathbf{A} z'_{0|t} = \mathbf{A} \mathbf{A}^\dagger \mathbf{y} + \mathbf{A} (\mathbf{I} - \mathbf{A}^\dagger \mathbf{A}) z_{0|t} = \mathbf{y} + \mathbf{0} = \mathbf{y}$.
(ii) The operator $\mathcal{P}_{\mathbf{A}, \mathbf{y}}(z) = \mathbf{A}^\dagger \mathbf{y} + (\mathbf{I} - \mathbf{A}^\dagger \mathbf{A}) z$ is affine and idempotent, and its residual $z - \mathcal{P}_{\mathbf{A}, \mathbf{y}}(z)$ lies in $\textit{Range}(\mathbf{A}^\top)$, which is orthogonal to $\textit{Null}(\mathbf{A})$. Hence, it is the orthogonal projection onto $\mathcal{M}_{\mathbf{A}, \mathbf{y}}$, meaning $z'_{0|t}$ is the closest point with $z_{0|t}$ within $\mathcal{M}_{\mathbf{A}}$.
(iii) $z_{0|t}$  can be decomposed as $z_{0|t}= \mathbf{A}^\dagger\mathbf{A} z_{0|t}+(\mathbf{I} - \mathbf{A}^\dagger \mathbf{A})z_{0|t}$, thus sharing the same null component with $z'_{0|t}$.
% \vskip -0.3in
\end{proof}

% 只要用Moore-Penrose 伪逆这个投影就是正交的，前提是y 必须属于 \mathbf{A}A 的列空间，但因为我们y=Az，所以必然成立，（ii）说明我们的投影不是任意的修正，而是从任意点 z_{0|t}到约束集合 \mathcal{M}_{\mathbf{A}, \mathbf{y}}（即所有满足 \mathbf{A} z = \mathbf{y}Az=y 的点）的 最近点， 也就是正交投影（orthogonal projection），最近即最小扰动-》保障视觉质量；(i) 已经说明“投影后满足一致性约束”，但 (ii) 提供了更强的理论保证，

\begin{proposition}[Random Observation Destroys Generation Fidelity]
\label{proposition:proposition1}
    Let $\mathbf{y} \sim \mathcal{N}(0, I)$ be a randomly sampled observation. Then the projected latent significantly deviates from the original trajectory and destroys the generation fidelity. 
\end{proposition}

\begin{proof}[Proof Sketch]
The $\mathbf{y}$ defines the constraint space:$\{ z_{0|t} : \mathbf{A} z_{0|t} \equiv \mathbf{y} \}$. During the generation process, $z'_{0|t}$ is used for denoising and image generation; hence, $\mathbf{y}$ directly influences final visual fidelity.
In this setting, the projection operator projects the $z_{0|t}$ into a random or noise space that is uncorrelated with the semantic manifold, directly destroying the denoising trajectory and degrading the generative quality.
% \vskip -0.3in
\end{proof}

\begin{proposition}[Trajectory-Intrinsic Observation Preserves Fidelity]
\label{proposition:proposition2}
    Let $\mathbf{y}$ be a trajectory-intrinsic observation and $\mathbf{A}$ satisfy the JL property, the projected latent remains close to the original denoising trajectory and preserves the fidelity.
\end{proposition}

\begin{proof}[Proof Sketch of Proposition~\ref{proposition:proposition2}]
Let $\Delta z = z'_{0|t}-z_{0|t}=\mathbf{A}^\dagger(\mathbf{A} z_{0|t} - \mathbf{y})$, since $\mathbf{A}$ satisfies the Johnson–Lindenstrauss  property, we have $\Delta z \approx z_{0|t}-z_{0|T_{proj}}$.
The DMs generates images by progressively denoising $\mathbf{z}_{0|t}$ from $t = T$ to $t = 0$ \citep{DBLP:conf/mm/FangCQM0C24},
thus $\mathbf{z}_{0|t}$ and $\mathbf{z}'_{0|t}$ nearly identical low-frequency semantics, and $\Delta z$ is dominated by high-frequency component. 
Therefore, the projected latent $\mathbf{z}'_{0|t}$ remains close to $\mathbf{z}_{0|t}$, preserving its semantic and structural coherence, thereby maintaining alignment with the original generative trajectory and ensuring both generative ability and watermark integrity.
\end{proof}

\begin{theorem}[CSGuard Simultaneously Satisfies Consistency and Fidelity Constraint]
\label{theorem:theorem1}
Let $\mathbf{A} \in \mathbb{R}^{M\times N}$ be a full-rank measurement matrix satisfying the JL property. Let $\mathbf{y}=\mathbf{A}z_{0|T_{proj}}$ be the trajectory-intrinsic observation, where $z_{0|T_{proj}}$ lies on the original denoising trajectory of the DMs at step $T_{proj}$.
For any $t \in [0,T_{proj}]$, CSGuard projects the latent $z_{0|t}$ as:
$z'_{0|t} = z_{0|t} - \mathbf{A}^\dagger(\mathbf{A} z_{0|t} - \mathbf{y})$. 
Then, the following key properties hold:
% \begin{enumerate}
%     \item Consistency Constraint: the projected latent $z'_{0|t}$ the minimal-perturbation solution to the CS constraint.
%     \item Fidelity Constraint: (i) the mapping $z_{0|t} \mapsto z'_{0|t}$ satisfies the minimal perturbation. (ii) The $z'_{0|t}$ remains close to the original generative trajectory and preserves the semantic coherence and watermark integrity.
% \end{enumerate}
\\ (1) Consistency Constraint: the projected latent $z'_{0|t}$ the minimal-perturbation solution to the CS constraint.\\
(2) Fidelity Constraint: (i) the mapping $z_{0|t} \mapsto z'_{0|t}$ satisfies the minimal perturbation. (ii) The $z'_{0|t}$ remains close to the original generative trajectory and preserves the semantic coherence and watermark integrity.
\end{theorem}
\begin{proof}[Proof Sketch]
(1) \textit{Consistency Constraint.}  
By construction,
$
\mathbf{z}_{0|t}' = \mathbf{z}_{0|t} - \mathbf{A}^\dagger (\mathbf{A} \mathbf{z}_{0|t} - \mathbf{y}).$
Multiplying both sides by $\mathbf{A}$ and using $\mathbf{A} \mathbf{A}^\dagger \mathbf{A} = \mathbf{A}$ (since $\mathbf{A}$ is full-rank), we obtain
$
\mathbf{A} \mathbf{z}_{0|t}' = \mathbf{A} \mathbf{z}_{0|t} - \mathbf{A} \mathbf{A}^\dagger (\mathbf{A} \mathbf{z}_{0|t} - \mathbf{y}) = \mathbf{y}.
$
(2) \textit{Fidelity Constraint.}  
(i) By Lemma~\ref{lemma:lemma1}, this projection is the minimum-norm solution to $\mathbf{A} \mathbf{z} \equiv \mathbf{y}$, thus minimizing $\|\mathbf{z} - \mathbf{z}_{0|t}\|_2$.
(ii) By Lemma~\ref{Lemma:Lemma2}, $\mathbf{A}$ approximately preserves this distance, and by Proposition~\ref{proposition:proposition2}, this ensures that $\mathbf{z}_{0|t}'$ remains close to the original generative trajectory and preserves the semantic coherence and watermark integrity.
\end{proof}

\begin{figure*}[tb]
  \begin{center}
    \centerline{\includegraphics[width=0.99\linewidth]{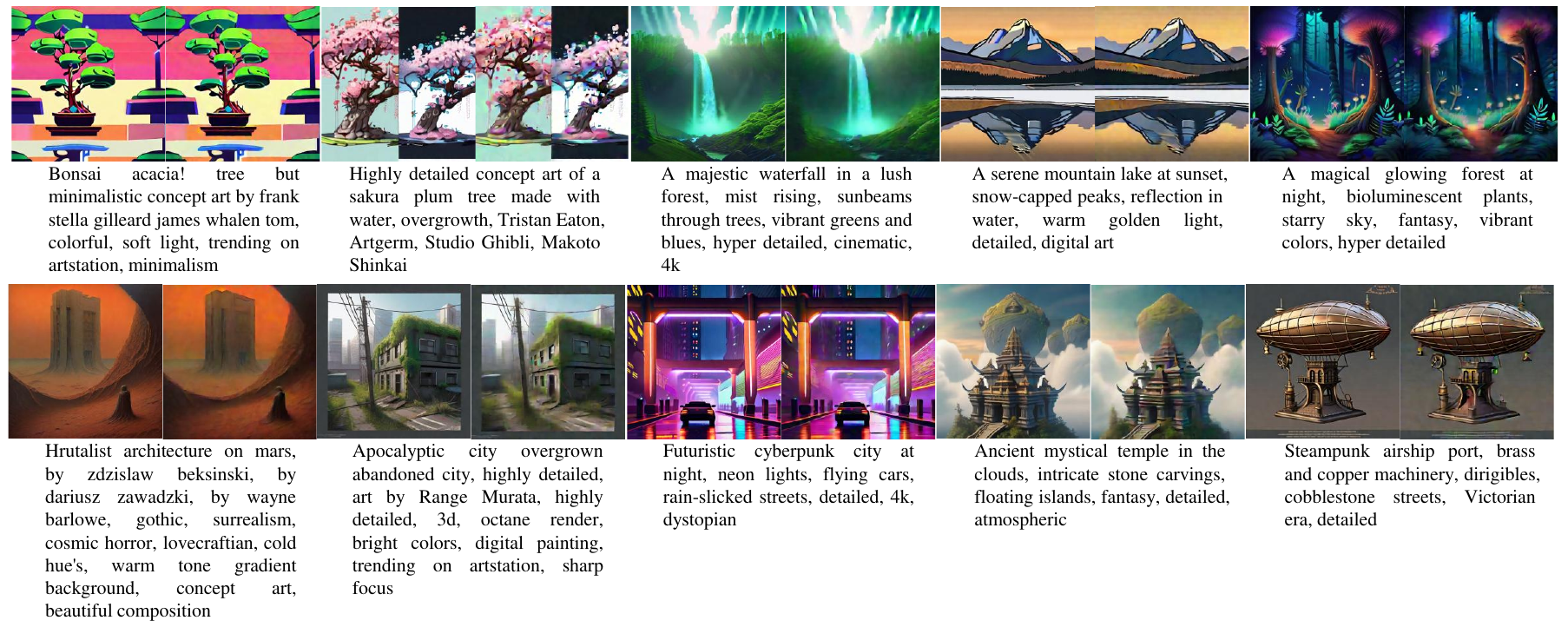}}
    \caption{
      Comparative visualization of watermarking results. For each prompt (bottom), the left image shows the watermarked output generated by the baseline method GaussShading (without CS constraints), while the right image presents the result produced by CSGuard. 
    }
    \label{fig:example}
  \end{center}
  \vskip -0.4in
\end{figure*}

\begin{figure}[tb]
  \begin{center}
    \centerline{\includegraphics[width=0.75\columnwidth]{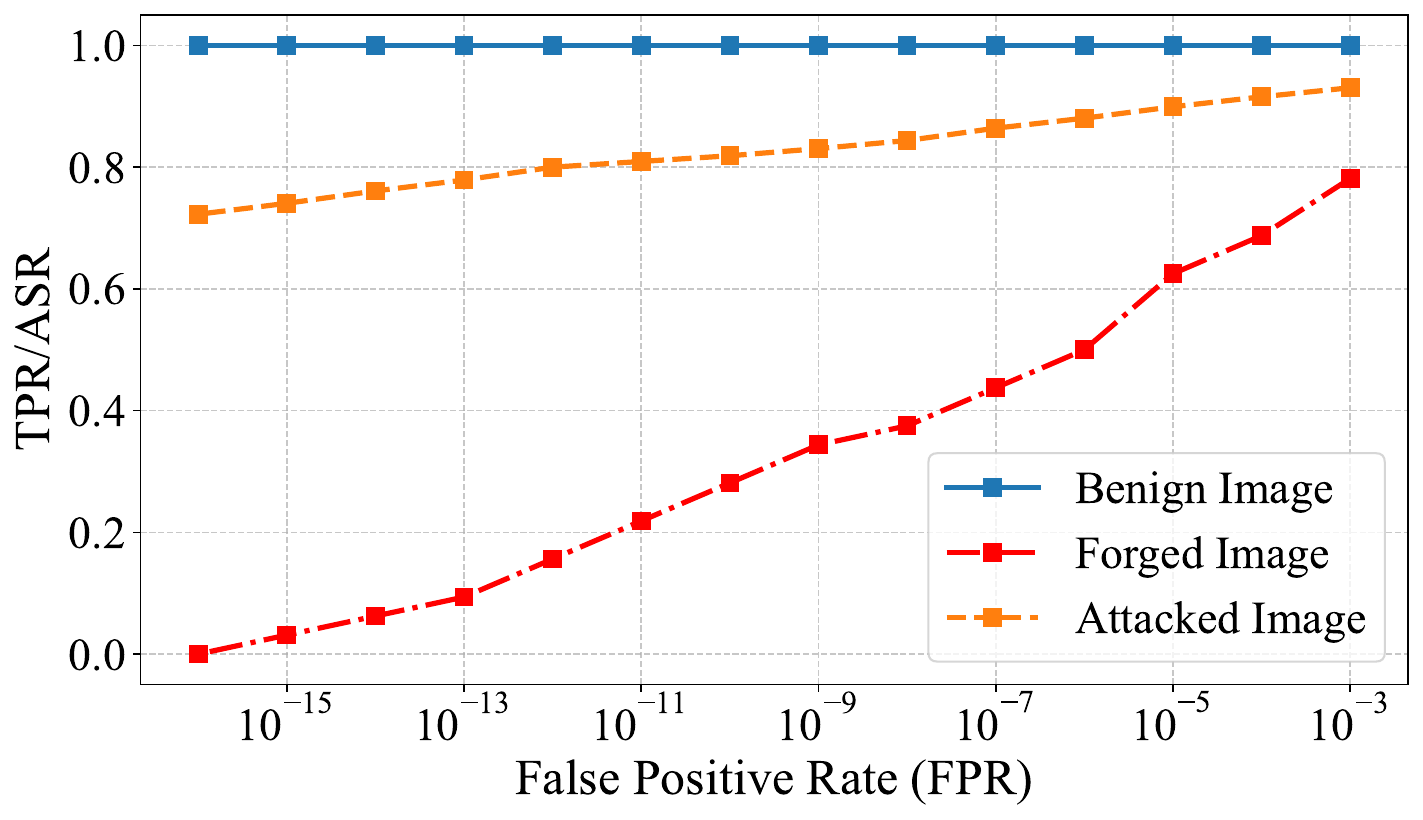}}
    \caption{TPR/ASR for benign and forged images under varying false positive rate (FPR) thresholds.}
    \label{fig:fpr}
  \end{center}
  \vskip -0.1in
\end{figure}

% \subsection{Details of Watermark Image Generation and Verfication Algorithms}

\label{app:alg}
\begin{algorithm}[tb]
\caption{Watermarked Image Generation}
\label{alg:generation}
\begin{algorithmic}[1]
\REQUIRE Secret matrix $\mathbf{A}\in \mathbb{R}^{M\times N}$, watermark $k$, prompt $c$, target step $T_{proj}$
\ENSURE Watermarked image $x^w$
\STATE Initialize watermarked noise: $z_T \leftarrow Sample(k)$
% \STATE Embed watermark: $z_T^w \leftarrow Embed(z_T,k)$
\FOR{$t = T$ \textbf{to} $1$}
    \STATE Predict noise: $\hat{z}_t \gets u_{\theta}(z_t, t, c)$
    \STATE Estimate clean image: $z_{0|t} \leftarrow \frac{(z_t-\sqrt{1-\alpha_{t}}\hat{z}_t)}{\sqrt{\alpha_t}}$
    \IF {$t = T_{\text{proj}}$}
        \STATE Construct observation: $\mathbf{y} \leftarrow \mathbf{A}z_{0|T_{proj}}$
    \ENDIF
    \IF{$t < T_{\text{proj}}$} 
        \STATE Project: $z'_{0|t} \leftarrow z_{0|t} - \mathbf{A}^\dagger(\mathbf{A} z_{0|t} - \mathbf{y})$
        \STATE $z_{0|t} \leftarrow z'_{0|t}$
    \ENDIF
    \STATE Denoise: $z_{t-1} \gets \sqrt{\alpha_{t-1}}\, z_{0|t} + \sqrt{1 - \alpha_{t-1}}\, \hat{z}_t$
\ENDFOR
\STATE Decode: $x^w \leftarrow \mathcal{D}(z_0)$
\STATE \textbf{return} $x^w$
\end{algorithmic}
\end{algorithm}

\begin{algorithm}[tb]
\caption{Watermark Verification}
\label{alg:verification}
\begin{algorithmic}[1]
\REQUIRE Image $x$, secret matrix $\mathbf{A} \in \mathbb{R}^{M\times N}$, watermark key $k$, target step $T_{\text{proj}}$
\ENSURE Boolean: whether $x$ is watermarked by Alice
\STATE Encode: $z_0 \leftarrow \mathcal{E}(x)$
% \STATE Initial $z_{0|0} \leftarrow \sqrt{\alpha_{t+1}}\, z_{0|t} + \sqrt{1 - \alpha_{t+1}}\, u(z_t, t)$
\STATE Reconstruct observation: $\mathbf{y} \leftarrow \mathbf{A} z_0$
\FOR{$t = 0$ \textbf{to} $T-1$}
    \STATE Predict noise: $\hat{z}_t \gets u_{\theta}(z_t, t)$
    \STATE Estimate clean image: $z_{0|t} \leftarrow \frac{(z_t-\sqrt{1-\alpha_{t}}\hat{z}_t)}{\sqrt{\alpha_t}}$
    \IF{$t < T_{\text{proj}}$}
        \STATE Project: $z_{0|t}' \leftarrow z_{0|t} - \mathbf{A}^\dagger (\mathbf{A} z_{0|t} - \mathbf{y})$
        \STATE $z_{0|t} \leftarrow z_{0|t}'$
    \ENDIF
    \STATE Invert: $z_{t+1} \leftarrow \sqrt{\alpha_{t+1}}\, z_{0|t} + \sqrt{1 - \alpha_{t+1}}\, \hat{z}_t$
\ENDFOR
\STATE Verify: $\text{result} \leftarrow Verify(z_T,k)$
\STATE \textbf{return} $\text{result}$
\end{algorithmic}
\end{algorithm}

\subsection{Security Analysis of the Secret Matrix}
\label{app:security_analysis}
The security of CSGuard fundamentally relies on the confidentiality of the secret projection matrix $\mathbf{A} \in \mathbb{R}^{M \times N}$, which serves as the core key for both embedding and verification. We analyze its security guarantees from two complementary perspectives: key generation properties, computational complexity against brute-force attacks.

\textbf{Key Generation and Randomness Guarantees.} The matrix is instantiated by sampling entries from a continuous random distribution (e.g., uniform or Gaussian) and is strictly maintained as a secret. This sampling strategy ensures two critical properties: (1) it provides unpredictability, preventing adversaries from modeling or reconstructing $\mathbf{A}$ through statistical inference; and (2) it satisfies the Restricted Isometry Property (RIP) required for compressed sensing, guaranteeing stable latent recovery during legitimate verification without compromising algorithmic stability.

\textbf{Key Space and Brute-Force Resistance.} The effective key space is governed by the continuous domain $\mathbb{R}^{M \times N}$. For a representative configuration (e.g., SDXL with latent dimension $N = 4 \times 64 \times 64 = 16,\!384$ and a compression ratio of 0.8, yielding $M \approx 13,\!107$), the matrix comprises over $2.15 \times 10^8$ independent real-valued entries. Even under a conservative binary discretization assumption, the exhaustive search space scales to $2^{2.15 \times 10^8}$, rendering brute-force key recovery computationally infeasible.
% Furthermore, the high dimensionality and continuous nature of the parameter space inherently resist dictionary-based or precomputation attacks, as partial key guesses yield statistically indistinguishable random projections.

Collectively, the exponential key space, continuous high-dimensional sampling establish a robust security foundation. These properties ensure that unauthorized watermark extraction, key recovery, or forgery attempts remain both computationally prohibitive and statistically bounded.

\subsection{Details of Experiments}
\label{app:experiment}

\subsubsection{Experimental Setting}

\paragraph{Metrics.}
In order to evaluate the performance of watermark traceability and the resistance to forgery attacks, we report the bit accuracy (BitAcc) both on the benign and forged images, and true positive rate (TPR) corresponding to the fixed false positive rate (FPR) on benign images and attack success rate (ASR) on forged images. 
In our main experiments, both the BitAcc and ASR of CSGuard are evaluated under a fixed FPR of $10^{-10}$. For baseline methods, we retain their original default settings.
% Tree-Rings uses a p-value threshold of 0.05; GaussShading sets FPR to $10^{-6}$, PRC uses $10^{-5}$, and GaussMarker adopts $10^{-2}$. 
Since RingID and WIND rely on distance-based metrics to determine watermark presence, we adhere to their original settings and report their detection rates on both benign and forged images.

For generative fidelity evaluation, we report CLIP scores\footnote{https://huggingface.co/openai/clip-vit-base-patch32}~\citep{DBLP:conf/icml/RadfordKHRGASAM21} to quantify the semantic alignment between generated images and their input prompts.
For generative fidelity evaluation, we report CLIP scores\footnote{https://huggingface.co/openai/clip-vit-base-patch32}~\cite{DBLP:conf/icml/RadfordKHRGASAM21} to quantify the semantic alignment between generated images and their input prompts. To provide a comprehensive assessment of generation quality in ~\ref{app:fidelity}, we further utilize CLIP Image Score and LPIPS to measure the perceptual and structural similarity between our watermarked outputs and the baseline images (generated without CS constraints). Additionally, we report the Fréchet Inception Distance (FID) to evaluate the distributional alignment between generated samples and real-world datasets.

\paragraph{Dataset.}
In our primary experiments, we adhere to the established forgery attack protocol~\citep{DBLP:conf/cvpr/0025LTFQ25}, employing prompts from the Stable-Diffusion-Prompt dataset\footnote{huggingface.co/datasets/Gustavosta/Stable-Diffusion-Prompts} for benign image generation and the Inappropriate Image Prompts dataset\footnote{huggingface.co/datasets/AIML-TUDA/i2p} for forged image synthesis.
To comprehensively evaluate generation quality in Appendix~\ref{app:fidelity}, we further utilize captions from the COCO dataset\footnote{https://cocodataset.org} for image generation. Additionally, real images from COCO serve as the reference distribution for computing both FID and CLIP Image Score, enabling a rigorous assessment of distributional alignment and semantic fidelity with the natural image manifold.

% \paragraph{Diffusion Models.}
% Following the setting within previous works \citep{DBLP:conf/cvpr/0025LTFQ25}, we use Stable Diffusion 2.1 (SD2.1)~\citep{DBLP:conf/cvpr/RombachBLEO22} as the attack model and 
% several common models as target models, namely 
% % Stable Diffusion 1.5 (SD1.5)~\citep{DBLP:conf/cvpr/RombachBLEO22},  
% Stable Diffusion XL(SDXL)~\citep{DBLP:conf/iclr/PodellELBDMPR24}, and PixArt-$\sum$~\citep{DBLP:conf/iclr/ChenYGYXWK0LL24}. All experiments use 512 $\times$ 512 images.
% % 说明PRC用的不同
\paragraph{Baseline and Diffusion Model.}
We conduct comparative evaluations against 6 latent-based watermarking schemes.
% , including 2 model-based schemes and 4 latent-based schemes. For model-based methods, we compare with Stable Signature\citep{DBLP:conf/iccv/FernandezCJDF23} and AquaLoRA \citep{DBLP:conf/icml/FengZH0WL00Y24}. For latent-based methods, 
Specifically, we compare with Tree-Rings \citep{DBLP:conf/nips/WenKGG23},RingID \citep{DBLP:conf/eccv/CiYSS24}, WIND \citep{DBLP:conf/iclr/ArabiFWHC25},  Gaussian Shading \citep{DBLP:conf/cvpr/YangZCF0Y24}, PRC-Watermark \citep{DBLP:conf/iclr/GunnZS25} and GaussMarker \citep{DBLP:conf/icml/LiHHH25}. 
Following the setting within previous works \citep{DBLP:conf/cvpr/0025LTFQ25}, We use Stable Diffusion 2.1 (SD2.1)~\citep{DBLP:conf/cvpr/RombachBLEO22} as the attack model and Stable Diffusion XL (SDXL)~\citep{DBLP:conf/iclr/PodellELBDMPR24} as the target model.
% for all methods except PRCMark, which uses SD1.5, 
Beyond the default setting described above, we further evaluate the generalizability of our method across diverse diffusion models and the impact of varying attack-target configurations in ~\ref{app:abaltion}. Specifically, we incorporate Stable Diffusion 1.5 (SD1.5)~\citep{DBLP:conf/cvpr/RombachBLEO22} and PixArt-$\Sigma$~\citep{DBLP:conf/iclr/ChenYGYXWK0LL24} as additional models to conduct a comprehensive analysis across different model families and architectural setups.
All experiments use 512 $\times$ 512 images.

\paragraph{Implementation Details.}
Our method is instantiated within the Gaussian Shading~\citep{DBLP:conf/cvpr/YangZCF0Y24} framework. Specifically, we adopt Gaussian Shading to synthesize the initial watermarked latent noise and to execute the final watermark verification step. Our proposed CS constraint is seamlessly integrated into both the forward generation and inverse DDIM sampling stages of this pipeline. 
For all primary experiments utilizing SDXL as the target diffusion model, we adopt default hyperparameter settings of a projection ratio of $0.4$ and a compression ratio of $0.8$. These values are selected to optimally balance forgery resistance with dual-fidelity preservation, as systematically analyzed in Section 5.4. All experiments are conducted on a single NVIDIA A40 GPU, ensuring consistent and reproducible computational conditions across all baselines and ablation studies.
% The code is available at \textbf{\textit{https://anonymous.4open.science/r/CSGuard-5740}}.

\subsubsection{Details of Assessment of Generation Fidelity.}
\label{app:fidelity}

To comprehensively evaluate generation fidelity beyond CLIP Score, we conducted both quantitative metric analysis and qualitative visual inspection. 

\paragraph{Quantitative Metrics.} 
To assess generation quality, we evaluate CSGuard across three complementary dimensions: distributional fidelity, semantic alignment, and perceptual consistency. 

We first compute FID scores against the real COCO dataset for both CSGuard-generated watermarked images and their non-watermarked counterparts. The FID difference between these two settings is only 2.48, indicating that our CS constraint introduces negligible additional distributional shift beyond the baseline generative model.
This distributional consistency is further corroborated by strong semantic alignment: complementary CLIP metrics computed against real COCO images yield an Image score of 2.889, while the CLIP Text score of 0.7536 confirms that generated images remain tightly aligned with input prompts.

Beyond aggregate metrics, we evaluate fidelity preservation across diverse stylistic domains by testing five distinct prompt categories (natural scenes, architecture, portraits, sci-fi/fantasy, and artistic styles). Compared to the baseline (Gaussian Shading without CS constraints), CSGuard maintains consistently high CLIP similarity (0.9504) and low LPIPS (0.1401) across all categories. This demonstrates that semantic content and perceptual quality are preserved uniformly, without degradation in any particular stylistic domain.

Collectively, these results form a coherent validation chain: the minimal FID gap guards against distributional drift, high CLIP scores ensure semantic fidelity, and low LPIPS captures perceptual nuances, together confirming that CSGuard's dual-fidelity design achieves high-quality generation without compromising watermark effectiveness.

\paragraph{Qualitative Visualization.} These quantitative findings are corroborated by visual inspections presented in Fig~\ref{fig:example}. All generated samples exhibit high visual fidelity, rich detail, and strong alignment with input captions, ranging from minimalist concept art to complex scenes involving lighting, texture, and compositional coherence. Crucially, no perceptible artifacts or degradation are introduced, confirming that our consistency-enforcing and dual-fidelity projection maintains the generative capacity of the underlying DMs. More importantly, watermarked images produced by CSGuard are visually indistinguishable from those of the baseline (without CS constraints), demonstrating that our minimal-perturbation consistency projection restricts CSGuard's influence on intermediate latents to negligible levels while keeping the generative process closely aligned with the original diffusion trajectory.

These results collectively validate that CSGuard achieves distributional consistency, semantic alignment, and perceptual quality without compromising the generative capability of pre-trained DMs.

\subsubsection{Additional Evaluation against Removal and Forgery.}
\label{app:add_attack}

\begin{table*}[t]
\centering
\caption{Robustness evaluation against watermark removal and forgery attacks. 
The table reports BitAcc and TPR/ASR under diverse adversarial settings. 
For removal attacks, higher BitAcc and TPR indicate stronger robustness. 
Conversely, for the forgery attack, lower BitAcc and ASR indicate superior defense.  
}
\label{tab:add_attack_eval}
\begin{small}
\begin{sc}
\scalebox{1.0}{
\begin{tabular}{c|cccc|c}
\toprule
Attack Method & DiffPure & VideoSeal & Regen(DM) & Regen(VAE) & WMCopier \\
\midrule
BitAcc(\%) & 78.83 & 63.82 & 80.79 & 82.38 & 53.5\\
TPR/ASR(\%) & 75.0 & 6.25   & 90.63 & 87.5  & 0.0\\
\bottomrule
\end{tabular}}
\end{sc}
\end{small}
% \vskip -0.1in
\end{table*}

In this section, we further evaluate the robustness of CSGuard against watermark removal attacks (including DiffPure~\citep{DBLP:conf/icml/NieGHXVA22}, VideoSeal~\citep{DBLP:journals/corr/abs-2510-20468}, and Regen~\citep{DBLP:conf/nips/ZhaoZSVGKVWL24}) and learning-based watermark forgery attacks (WMCopier~\citep{dong2025wmcopier}). Experimental results demonstrate that our method exhibits strong robustness against both removal and forgery attacks. Detailed quantitative results are presented in Table~\ref{tab:add_attack_eval}.

\textbf{DiffPure~\citep{DBLP:conf/icml/NieGHXVA22}}: This attack removes watermark signals by denoising during the diffusion process while attempting to preserve semantic content. As a latent-based scheme, CSGuard embeds watermark signals that are intrinsically coupled with the semantic representations of the image, rather than superimposing them as superficial perturbations; consequently, it exhibits inherent robustness. Empirical results show that CSGuard retains a BitAcc of 78.83\% and a TPR of 75\% after DiffPure processing, confirming its resistance to semantics-preserving removal attacks.

\textbf{VideoSeal~\citep{DBLP:journals/corr/abs-2510-20468}}: This attack is designed primarily for post-hoc watermark removal. We observe that effective watermark suppression via VideoSeal incurs severe semantic degradation: the CLIP similarity score drops from 0.3135 to 0.2485, indicating substantial utility loss. This trade-off renders the attack impractical for adversaries seeking to preserve image quality.

\textbf{Regen~\citep{DBLP:conf/nips/ZhaoZSVGKVWL24}}: This attack eliminates watermarks by reconstructing watermarked images using generative models. We evaluated CSGuard against this attack using both diffusion models and non-diffusion architectures (VAE), achieving TPRs of over 90.63\% and 87.5\% against diffusion-based and VAE-based regeneration attacks, respectively.

\textbf{WMCopier~\citep{dong2025wmcopier}}: This attack attempts to estimate and forge watermarks by collecting a large volume of watermarked images to learn the distribution of watermarked content. CSGuard effectively resists the WMCopier attack, reducing the bit accuracy of forged images to 0.535 (equivalent to random guessing) and achieving a 0\% ASR. This confirms that CSGuard maintains robustness even when the attacker attempts to fit the specific distribution of our watermarked outputs.

Collectively, these evaluations demonstrate that CSGuard delivers comprehensive robustness against both watermark removal and forgery threats, providing reliable and verifiable watermark protection under diverse and realistic adversarial scenarios.

\subsubsection{Additional Analysis Experiment}
\label{app:abaltion}

\paragraph{Analysis on FPR.}
As shown in Fig~\ref{fig:fpr}, CSGuard achieves perfect detection on benign images across all FPR thresholds. TPR remains exactly 1.00. For distorted images, the average TPR exceeds 70\% across the entire FPR range, demonstrating strong robustness to post-processing. While ASR naturally increases as the FPR threshold relaxes (due to a more permissive detector), our default operating point at FPR = $10^{-10}$  strikes a practical balance: it maintains a high TPR of about 80\% on distorted images while suppressing ASR below 30\%, thereby enabling CSGuard reliable forensic verification against forgery attack.

\paragraph{Analysis on Different DMs.} 
To further validate the generalizability of our method across diverse DMs, we evaluate the watermarking performance and image generation quality of CSGuard on four representative DMs: SDXL, PixArt~\citep{DBLP:conf/iclr/ChenYGYXWK0LL24}, SD1.5, and SD2.1. In addition to adhering to the baseline forgery attack setup~\citep{DBLP:conf/cvpr/0025LTFQ25}, which employs SD2.1 as the proxy attacker, we further analyzed CSGuard's forgery resistance under varying proxy-target model pairs.

\textit{Generalizability across Diffenert DMs}. 
CSGuard demonstrates robust applicability across different DMs, maintaining excellent watermarking performance and image quality. As shown in Table~\ref{tab:wm_model}, our method consistently achieves high watermark effectiveness on benign images across all models: bit accuracy exceeds 96.1\%, TPR remains 100.0\%, and CLIP scores stay near 0.31. These results demonstrate that the dual-fidelity design effectively retains both the generative capability of the original DMs and the effectiveness of the watermarking schema.

\textit{Forgery Resistance under Cross-architecture scenarios.}
Regarding forgery resistance, CSGuard maintains strong robustness when the target and proxy models differ architecturally. For instance, attacking PixArt-generated images with SD2.1 yields an ASR of only 6.25\%. As detailed in Table~\ref{tab:cross_model_attack}, similar robustness is observed in other cross-architecture scenarios, where the ASRs for PixArt-SDXL, PixArt-SD2.1, and SDXL-PixArt are 37.5\%, 21.87\%, and 31.25\%, respectively. These results confirm that binding generation-inversion symmetry to a secret matrix effectively thwarts unauthorized watermark forging in cross-model settings.

\textit{Vulnerable to Model-Identical Attacks.}
However, when the attacker employs a proxy model from the same family for both inversion and forgery, the ASR increases notably. For example, the ASR rises significantly to 87.50\% and 93.75\% when the attacker is SD2.1 and the target model are SD1.5 and SD2.1, respectively. This stems from the model-matching attack: the adversary uses SD2.1 for both inversion and regeneration, yielding trajectories highly aligned with CSGuard's own diffusion path. While our method preserves the original trajectory to maintain fidelity and quality, it cannot arbitrarily perturb the latent dynamics, thus leaving a vulnerability when the attacker and defender share the same model family. This highlights a fundamental trade-off: strict trajectory consistency enables high fidelity but limits robustness under model-identical attacks. Addressing this limitation is a key direction for future work.

\textit{Vulnerability Mitigation via Distance-Based Detection.} To mitigate this vulnerability, we implemented a post-hoc detection protocol inspired by PAI~\citep{DBLP:journals/corr/abs-2601-06639}, which evaluates the divergence between the extracted latent and the reference watermark to identify potential forgeries. This demonstrates that CSGuard can be straightforwardly extended to forensic scenarios for verifying image authenticity. As detailed in Table~\ref{tab:cross_model_attack}, this extended detection mechanism achieves a perfect Discrimination Accuracy (DisAcc) of 100.0\% in cross-architecture scenarios (e.g., SD2.1 attacking PixArt), effectively distinguishing benign images from forgeries. Crucially, even under the challenging Model-Identical setting (e.g., SD2.1 $\rightarrow$ SD2.1), the method maintains a DetAcc of 79.68\%. This confirms that despite the trajectory alignment increasing ASR, the latent distributions of benign and forged images remain sufficiently distinguishable to support reliable forensic verification.
Nevertheless, characterizing the  limits of trajectory-based detection and developing architecture-agnostic verification protocols constitute promising avenues for future research.

\begin{table*}[tb]
\centering
\caption{Watermarking performance under varying diffusion models. 
}
\label{tab:wm_model}
\begin{small}
\begin{sc}
\scalebox{0.86}{
\begin{tabular}{l|c|cccc|cc}
\toprule
\multicolumn{2}{c}{} &\multicolumn{4}{|c|}{Benign Image}  &\multicolumn{2}{c}{Forged image} \\ \midrule
Method &{DMs} & {BitAcc (\%)$\uparrow$} & {TPR (\%)$\uparrow$} & {CLIP Score$\uparrow$} &{TimeRatio $\downarrow$}& {BitAcc (\%)$\downarrow$} & {ASR (\%)$\downarrow$}\\
\midrule
\multirow{4}{*}{CSGuard} &SDXL   & 96.15      & 100.0      & 0.3135 &1.005$\times$ &$70.26_{25.89}$       &28.12 \\
 &PixArt & 96.25      & 100.0      & 0.3104 &1.005$\times$ & $62.91_{33.34}$       & 6.25\\
 &SD1.5 & 99.52      & 100.0      & 0.3135 &1.005$\times$ &$92.87_{6.65}$       &87.50\\
 &SD2.1 & 99.29     & 100.0    & 0.3179 &1.007$\times$ & $92.02_{7.27}$       & 93.75 \\
\bottomrule
\end{tabular}}
\end{sc}
\end{small}
% \vskip -0.1in
\end{table*}

\begin{table*}[tb]
\centering
\caption{Watermarking performance under varying attack-target configurations.  
}
\label{tab:cross_model_attack}
\begin{small}
\begin{sc}
\scalebox{1.0}{
\begin{tabular}{l l c c}
\toprule
Attack DM & Target DM & ASR (\%)$\downarrow$ & DisAcc (\%)$\uparrow$ \\
\midrule
SD2.1 & SDXL & 28.12 & 100.0 \\
SD2.1 & PixArt & 6.25 & 100.0 \\
SD2.1 & SD1.5 & 87.5  & 79.68 \\
SD2.1 & DS2.1 & 93.75 & 79.68 \\
SDXL & PixArt & 31.25 & 100.0 \\
PixArt & SDXL &37.50 & 100.0 \\
PixArt & SD2.1 & 21.87 & 100.0 \\
% SDXL & SD2.1 & 65.60 \\
\bottomrule
\end{tabular}}
\end{sc}
\end{small}
% \vskip -0.1in
\end{table*}

\paragraph{Analysis on Different Secret Matrix.}
We investigate the feasibility of employing publicly-known degradation operators as the secret matrix $\mathbf{A}$ in our watermarking framework. Specifically, drawing inspiration from image restoration literature~\citep{DBLP:conf/iclr/WangYZ23,DBLP:conf/cvpr/GarberT24} we implement three representative operators: 
(1) \textit{Pooling}: employs average pooling for spatial downsampling with nearest-neighbor interpolation for reconstruction. (2) \textit{Blur}: applies Gaussian blurring with frequency-domain Wiener filtering for deconvolution. and (3) \textit{Degradation}:  combines Gaussian blurring followed by downsampling, approximating its pseudoinverse via Tikhonov regularization.

% (1) \textit{Pooling}: We construct the matrix $\mathbf{A}$ via average pooling to perform content-preserving spatial downsampling, and approximate its pseudoinverse $\mathbf{A}^\dagger$ using nearest-neighbor interpolation for resolution restoration.
% (2) \textit{Blur}: We construct the matrix $\mathbf{A}$ via Gaussian blur, and the corresponding pseudoinverse $\mathbf{A}^\dagger$ is approximated using Wiener filtering in the frequency domain.
% (3) \textit{Degradation}: We construct the matrix $\mathbf{A}$ via  Gaussian blurring followed by downsampling, and its pseudoinverse $\mathbf{A}^\dagger$ serves as the super-resolution reconstruction operator, approximated using Tikhonov regularization.

\begin{figure}[tb]
  \begin{center}
    \centerline{\includegraphics[width=\columnwidth]{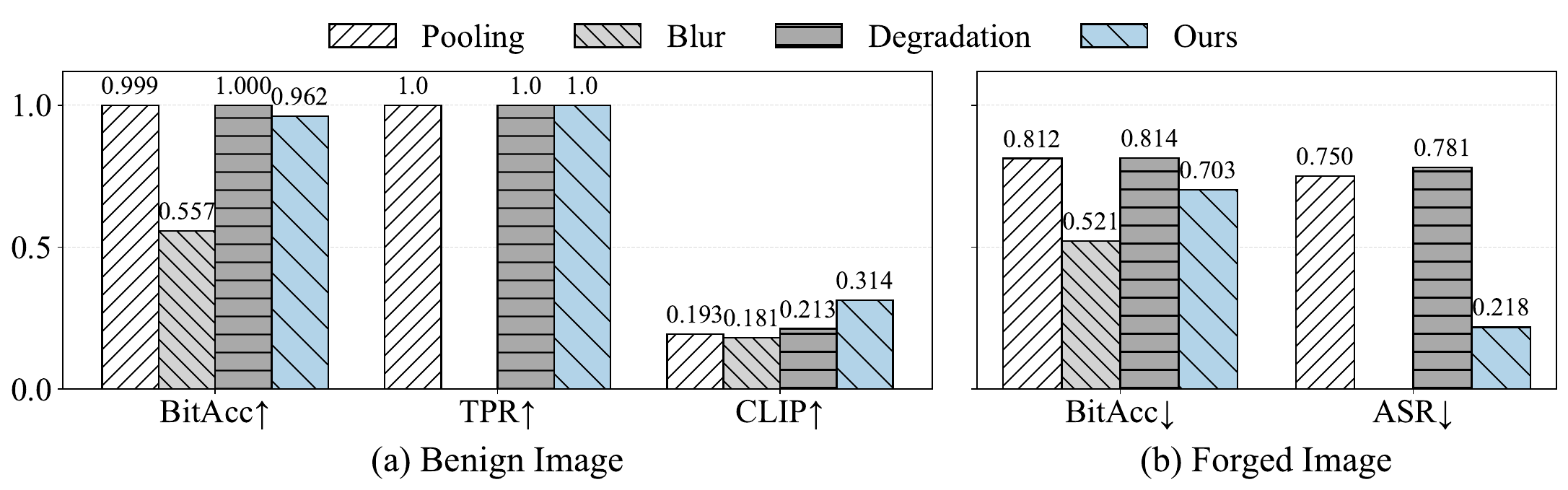}}
    \caption{TPR/ASR for benign and forged images under varying matries.}
    \label{fig:diff_A}
  \end{center}
  \vskip -0.1in
\end{figure}

\textit{Impact on Dual Fidelity}. On benign images, while public operators achieve high bit accuracy and perfect TPR, they severely compromise images' semantic quality, as evidenced by substantially lower CLIP scores (lower than 0.213) compared to our method. This significant degradation indicates that public operators introduce substantial distortions, fundamentally altering the semantic content of generated images. Crucially, this violates the dual-fidelity constraint essential for practical deployment.

\textit{Vulnerability to Forgery Attack}. These public operators fail to provide adequate resistance against forgery attacks. As shown in Fig~\ref{fig:diff_A}, watermarked images protected by public operators exhibit significantly higher bit accuracy under forgery conditions ( more than 0.812 for Pooling and Blur). Consequently, the ASR remain unacceptably high (more than0.750), whereas our method achieves a substantially lower rate of 0.218. 
Regarding the Degradation operator, it yields a bit accuracy of approximately 0.5 on forged images, leading to an ASR of 0. However, this low ASR does not imply effective resistance; instead, it confirms that the operator has completely obliterated the watermark information.

\textit{Security Analysis}. The fundamental weakness stems from their transparency: since the degradation process is publicly known, adversaries can readily construct corresponding inverse operators to reverse the protection mechanism. This violates the core security principle that the secret matrix $\mathbf{A}$ must remain private to prevent unauthorized users from exploiting the generation–inversion symmetry. 
% Our learned secret matrix addresses this vulnerability by keeping the transformation private, thereby binding the watermark embedding and verification to authorized users only.

\subsection{Additional Discussion}
\label{app:discussion}

\paragraph{Impact on Gaussianity and Generative Distribution.}
% The integration of CS constrain into DMsMark introduces an inherent tension between security constraints and preservation of the model's native generative distribution. 
While CSGuard imposes a theoretical constraint on the latent space, its design principles ensure that the semantic coherence and statistical properties of the original generative distribution are empirically preserved.

\textit{Theoretical Constraints on Latent Distribution.}
The enforcement of the CS consistency constraint, defined as $Az_{0|t} \equiv y$, restricts the intermediate latent variables $z_{0|t}$ to an affine constraint set $\mathcal{M}_{A,y}$. 
% As formalized in Lemma~\ref{lemma:lemma1}, the minimal-perturbation consistency projection modifies the component of the latent variable lying within the range space of the secret matrix $\mathbf{A}$. 
Consequently, in the subspace spanned by $\text{Range}(A)$, the latent distribution deviates from the unconditional Gaussian prior assumed by the standard denoising process. 
However, CSGuard mitigates this theoretical deviation through two key mechanisms. First, Corollary~\ref{coro:coro1} establishes that the projection operator leaves the null-space component of $z_{0|t}$ unchanged. 
Since the null space retains full flexibility, the model preserves its capacity to optimize latent representations for semantic coherence without violating the CS constraint. Second, the trajectory-intrinsic observation construction ensures that the observation vector $\mathbf{y}$ is derived from the original denoising path. 
This design guarantees that the projected latent $z'_{0|t}$ remains close to the original generative trajectory (Theorem~\ref{theorem:theorem1}), thereby preventing the significant distributional shift observed when using random observations (Proposition~\ref{proposition:proposition1}).

\textit{Empirical Preservation of Generative Characteristics.}
Empirically, the impact of these constraints on the generative distribution is negligible. As reported in Table~\ref{tab:wm_result}, CSGuard achieves a CLIP score of 0.3135, indicating that the semantic alignment between the generated images and input prompts remains intact, suggesting that the effective data manifold is not distorted. Furthermore, the FID difference between watermarked and non-watermarked images is only 2.48, and across five diverse stylistic categories, CSGuard maintains high CLIP similarity (0.9504) and low LPIPS (0.1401) compared to the baseline. Visual comparisons in Fig~\ref{fig:example} corroborate this, showing that watermarked images are visually indistinguishable from baseline outputs, with no perceptible artifacts introduced by the projection process.

\paragraph{Potential for Plug-and-Play Integration.} 
Beyond standalone deployment, CSGuard is architected with modularity in mind, suggesting potential as a lightweight, plug-and-play enhancement for existing latent-based watermarking schemes. 

Our integration experiments with representative methods such as Tree-Rings, GaussMarker, and PRCMark indicate promising directions for improving forgery resistance without altering their core pipelines. For instance, when layered atop Tree-Rings, the statistical significance (p-value) on forged images rises from 0 to 0.2585, crossing the conventional detection threshold of 0.05. Similarly, integration with GaussMarker and PRCMark suggests a widening of the bit accuracy gap between benign and forged samples by approximately 0.1848 and 0.2534, respectively. These observations imply that enforcing CS constraints can break the generation-inversion symmetry exploited by adversaries, even when applied to distinct watermarking backbones.

While these initial results highlight the method's extensibility, we acknowledge that seamless integration with learnable frameworks (e.g., GaussMarker) may require careful architectural adaptations to optimize the integration of CS constraints. 
Nevertheless, these findings underscore CSGuard’s potential as a versatile add-on, and we plan to systematically explore its compatibility and optimal coupling strategies with diverse baselines as a key direction for future work.

\paragraph{Broader Impact.}
CSGuard addresses a critical security vulnerability in latent-based diffusion model watermarking, significantly enhancing the reliability of AI-generated content attribution. By breaking the generation–inversion symmetry exploited by forgery attacks, it prevents malicious actors from falsely attributing deepfakes or harmful content to innocent users. This strengthens intellectual property protection, mitigates misinformation campaigns, and supports trustworthy digital forensics. Furthermore, its training-free design and negligible computational overhead (<1\%) enable practical deployment without restricting access to open generative models. 
The method does not introduce new generative models or datasets and has
no negative impact on society.

%%%%%%%%%%%%%%%%%%%%%%%%%%%%%%%%%%%%%%%%%%%%%%%%%%%%%%%%%%%%

% \newpage
% \input{checklist.tex}

\end{document}